\documentclass{article}

\usepackage[final]{staix_2026}         

\usepackage[utf8]{inputenc}
\usepackage[T1]{fontenc}
\usepackage[
  colorlinks=true,
  linkcolor=linkblue,
  citecolor=linkblue,
  urlcolor=linkblue
]{hyperref}
\usepackage{url}
\usepackage{booktabs}
\usepackage{amsfonts}
\usepackage{amsmath,amssymb,amsthm}
\usepackage{nicefrac}
\usepackage{microtype}
\usepackage{tabularx}
\usepackage{xcolor}
\definecolor{linkblue}{RGB}{0,0,150}\usepackage{graphicx}
\usepackage{multirow}
\usepackage{algorithm}
\usepackage{algorithmicx}
\usepackage{algpseudocode}
\usepackage[capitalize,nameinlink]{cleveref}
\newtheorem{proposition}{Proposition}
\newtheorem{corollary}{Corollary}
\newtheorem{assumption}{Assumption}

\newcommand{\E}{\mathbb{E}}
\newcommand{\Var}{\mathrm{Var}}
\newcommand{\R}{\mathbb{R}}
\newcommand{\ST}{\mathrm{ST}}
\newcommand{\IGSD}{\textsc{IGSD}}
\newcommand{\swap}{\mathrm{swap}}
\newcommand{\zero}{\mathrm{zero}}
\newcommand{\logit}{\mathrm{logit}}

\title{Beyond Importance: Interchange-Sobol Sensitivity Reveals Task-Specific Content Channels in Transformer Components}

\author{%
  Yifeng Guo \\
  St. Jude Children's Research Hospital \\
  Memphis, TN \\
  \texttt{yifeng.guo@stjude.org}
  \And
  Jin-Hong Du \\
  The University of Hong Kong \\
  Hong Kong, CN \\
  \texttt{jinhongd@hku.hk}
  \And
  Xiang Chen \\
  St. Jude Children's Research Hospital \\
  Memphis, TN \\
  \texttt{xiang.chen@stjude.org}
}

\begin{document}

\maketitle

\begin{abstract}
Mechanistic interpretability methods summarize a transformer component by a single importance score, conflating two distinct roles: a component may matter because it transports task-relevant content, or because the forward computation degrades when its contribution is removed. We introduce \emph{Interchange-Group Sobol Decomposition} (IGSD), a paired-intervention framework that compares matched activation replacement with zero ablation on the same component, estimates two Sobol-style variance indices, and uses their signed difference to separate the two roles, with intervention validity monitored by a symmetric off-manifold diagnostic $\widehat{\mathrm{ST}}>1$. In factual recall, IGSD identifies an early-layer content channel in both GPT-2 small and Qwen2.5-1.5B that standard importance methods underestimate. A controlled subject and relation donor design shows that the early channel transports relation-frame content while late attention transports subject-retrieval content, refining at head granularity to the known $\mathrm{Attn}_{L9H8}$ head. Late-layer clamping confirms that the early signal is expressed through downstream transformations rather than residual pass-through. These results show that replacement and deletion are not interchangeable controls and their divergence provides a practical statistical diagnostic for content transport in transformer components.

\end{abstract}


\section{Introduction}
\label{sec:intro}

Mechanistic interpretability is central to trustworthy AI because it asks not only whether a transformer component matters, but what computation it performs and how that computation supports model behavior \citep{bereska2024mechanistic, somvanshi2026bridging}. Standard tools, including direct logit attribution (DLA)~\citep{elhage2021mathematical,wang2023interpretability}, ablation-based component scoring~\citep{wang2023interpretability,conmy2023towards,zhang2024towards}, and gradient-based attribution such as AtP*~\citep{kramar2024atp}, typically summarize each component by a single scalar importance score. This scalar is useful for localization, but it mixes two different mechanisms. A component may matter because it transports task-relevant \emph{content}, so replacing its activation with another prompt's activation changes the model's answer. It may also matter because it provides a computational \emph{substrate}, so removing it damages the forward computation. These two roles are not statistically equivalent, and in retrieval tasks they can lead to sharply different component rankings.

Figure~\ref{fig:teaser} illustrates this distinction on factual recall using CounterFact~\citep{meng2022locating}. For the prompt ``Yahoo!\ Tech is owned by\,$\_\_$'', GPT-2 small assigns high probability to the correct answer ``Yahoo''. Standard importance scores do not place the early MLP among the leading components. However, replacing the output of $\mathrm{MLP}_{\mathrm{L0}}$ with a matched donor activation from the prompt ``Cologne Carnival can be found in\,$\_\_$'' changes the prediction pattern and demotes ``Yahoo'' from rank~\#1 to rank~\#95. The intervention does not simply remove computation. It injects alternative content, and the resulting failure indicates that the early component is content-sensitive in a way that deletion-style importance scores understate. This example suggests that replacement and deletion should not be treated as interchangeable controls.

\begin{figure}[t]
  \centering
  \includegraphics[width=\linewidth]{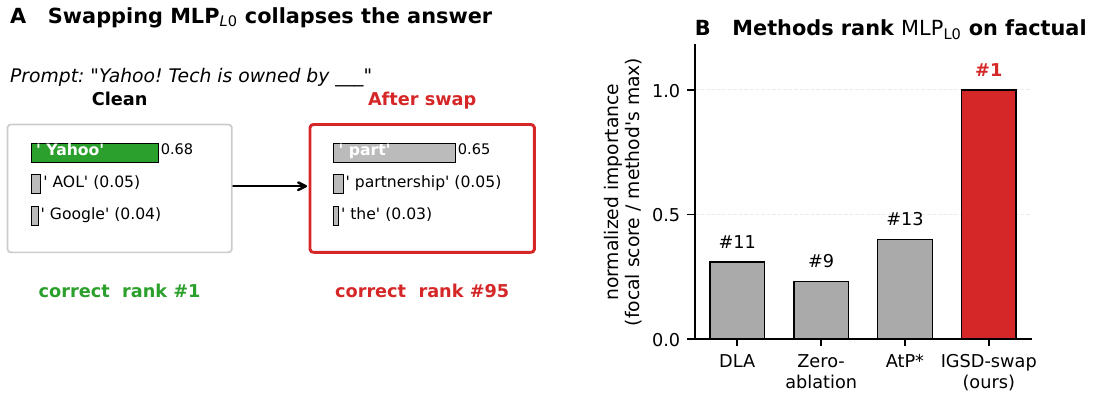}
  \caption{\textbf{Motivating phenomenon.} \textbf{(A)} Swapping $\mathrm{MLP}_{\mathrm{L0}}$ at the answer position with a matched donor activation changes the predicted token distribution and demotes the correct answer. \textbf{(B)} On factual recall, standard importance methods rank $\mathrm{MLP}_{\mathrm{L0}}$ between \#9 and \#13 of 24 layer-local groups, whereas \IGSD{} ranks it \#1.}
  \label{fig:teaser}
\end{figure}

We introduce \emph{Interchange-Group Sobol Decomposition} (\IGSD{}), a paired-intervention framework that compares two perturbations of the same component under a common variance scale. For a layer-local group, IGSD estimates one Sobol-style index from matched-donor activation replacement and another from zeroing the same activation. Their signed difference is positive when incorrect content is more disruptive than absence and negative when absence is more disruptive than alternative content. The contrast therefore separates content-sensitive channels from deletion-sensitive substrates as a property of the component, task, and intervention distribution. A controlled donor design over subject and relation then identifies which content factor is transported by the swap.

This paper makes three contributions. (1) We formulate IGSD as a paired intervention framework for transformer components, together with matched-pair inference and a factorial donor design. The accompanying theory gives a mis-specification bound for matched interchange (Proposition~\ref{prop:main}), paired-bootstrap inference for 
$\widehat\delta(g)$ (Proposition~\ref{prop:delta-clt}), and two identification results for content-factor and latent-role interpretation (Propositions~\ref{prop:identify} and~\ref{prop:role-identify}).
 (2) Across different tasks, IGSD finds localized task-specific content-transport profiles at the answer-aligned end-token position, complementary to subject-token causal tracing~\citep{meng2022locating,geva2023dissecting} rather than a substitute for it. In factual recall, it identifies an early-layer channel in both GPT-2 small~\citep{radford2019language} and Qwen2.5-1.5B~\citep{hui2024qwen2} that is underestimated by standard importance scores. At finer granularity, the same procedure recovers the known factual retrieval head $\mathrm{Attn}_{\mathrm{L9H8}}$~\citep{meng2022locating,geva2023dissecting}. (3) We provide downstream validation via late-layer clamping and group masking, confirming that the identified factual components are transformed by later layers and can support pruning-based fast inference.

\section{Related Work and Setup}
\label{sec:background}

\subsection{Related work}

\paragraph{Intervention-based interpretability.}
Activation patching replaces an internal activation with one from another forward pass and is now a standard intervention in mechanistic interpretability~\citep{vig2020investigating,meng2022locating,wang2023interpretability}. Causal abstraction work formalizes such interventions, together with path patching, mediation, scrubbing, and alignment search through interchange interventions~\citep{geiger2025causal}. Recent critiques show that accuracy-based interchange certificates can be too permissive without additional restrictions on the alignment map~\citep{sutter2026the}. \IGSD{} uses the same interchange primitive, but targets variance-based component sensitivity and pairs replacement with deletion on the same component.

\paragraph{Factual recall mechanisms.}
Prior work has localized factual associations to transformer internals. \citet{meng2022locating} identify mid-layer MLPs at subject-token positions, \citet{geva2023dissecting} describe a three-stage recall circuit involving subject enrichment, relation propagation, and late attribute extraction, and \citet{hernandez2024linearity} show that relations are linearly decodable from subject representations. \IGSD{} studies a complementary end-token estimand under matched interchange, allowing us to ask whether a component transports subject content, relation content, or both.

\paragraph{Content-aware attribution.}
Several recent studies argue that resample or interchange ablations can reveal information content missed by zero or mean ablation~\citep{heimersheim2024use,zhang2024towards}. \IGSD{} turns this qualitative asymmetry into a signed variance-based diagnostic. We compare against AtP*~\citep{kramar2024atp} as a strong gradient-based importance baseline, but our target is different. We contrast two intervention distributions to separate content sensitivity from deletion sensitivity. This is complementary to circuit-level and feature-dictionary approaches such as automated circuit discovery, sparse autoencoders, and transcoder-based circuit analysis~\citep{conmy2023towards,bricken2023towards,templeton2024scaling,golimblevskaia2026circuit}. Methodologically, \IGSD{} adapts variance-based sensitivity analysis~\citep{sobol2001global,jansen1999analysis,saltelli2010variance} to transformer components through matched-pair interchange.

\subsection{Setup}
\label{sec:setup}

We use the standard pre-norm transformer block, instantiated by GPT-2 small~\citep{radford2019language} and Qwen2.5-1.5B~\citep{hui2024qwen2} up to the normalization choice; architecture details are summarized in Appendix~\ref{app:architectures}. For an $L$-layer model, the residual stream $r_\ell \in \mathbb{R}^{d}$ evolves as
\begin{align}
\tilde{r}_\ell &= r_\ell + \mathrm{Attn}_\ell(\mathrm{LN}_1(r_\ell)),\\
r_{\ell+1}    &= \tilde{r}_\ell + \mathrm{MLP}_\ell(\mathrm{LN}_2(\tilde{r}_\ell)),
\label{eq:transformer-block}
\end{align}
where $\mathrm{LN}$ is LayerNorm in GPT-2 small and RMSNorm in Qwen2.5-1.5B. We partition the model into $K=2L$ layer-local groups, one attention group and one MLP group per layer. Their written residual-stream vectors at token position $t$ are
\begin{equation}
h_{\mathrm{Attn}_\ell}(x,t)
:= \mathrm{Attn}_\ell(\mathrm{LN}_1(r_\ell(x)))[t],
\qquad
h_{\mathrm{MLP}_\ell}(x,t)
:= \mathrm{MLP}_\ell(\mathrm{LN}_2(\tilde r_\ell(x)))[t],
\label{eq:hg-def}
\end{equation}
both in $\mathbb{R}^{d}$. \IGSD{} intervenes on $h_g(x,t)$ for $g\in\{\mathrm{Attn}_\ell,\mathrm{MLP}_\ell\}$ at the answer-aligned target token, as shown in Figure~\ref{fig:method}. Because the block update is additive, replacing one written vector changes only that component's contribution while leaving the residual-stream carry-through intact. This distinction is important when interpreting late-layer clamping in Section~\ref{sec:experiments-clamp}.

For each prompt $x$, we use the scalar logit-margin output
\begin{equation}
Y(x)=\logit(\mathrm{correct}\mid x)-\logit(\mathrm{foil}\mid x).
\end{equation}
For each base prompt $x_A$, we choose a swap partner $x_B$ from the ten nearest neighbors of $x_A$ in a standardized feature space combining $Y$ and per-layer total DLA magnitudes. This matching keeps donor-recipient pairs closer to the empirical activation manifold while preserving variation in the target component's content.

The matched donor rule should be understood as defining the intervention distribution for IGSD, not as claiming exact propensity-score balance over all prompt-side covariates. Its primary purpose is to keep the donor activation close to the empirical activation manifold by matching output difficulty and broad attribution scale, while still allowing task-relevant content to vary. Because prompt syntax, entity type, or factual domain may also affect the swap response, we diagnose residual imbalance and report reweighted robustness analyses in Appendix~\ref{app:donor-fidelity}.

\begin{figure}[t]
  \centering
  \includegraphics[width=\linewidth]{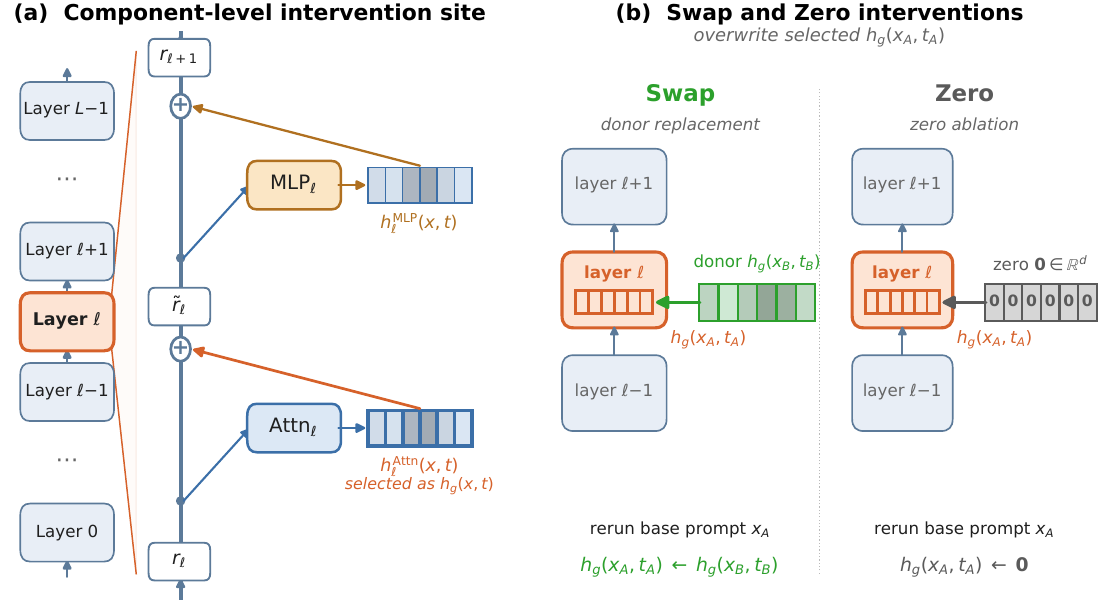}
  \caption{\textbf{\IGSD{} method.} \textbf{(a)} At layer $\ell$, attention and MLP modules write additive vectors into the residual stream; \IGSD{} targets one written vector $h_g(x,t)$. \textbf{(b)} For the same base prompt $x_A$, the swap intervention replaces $h_g(x_A,t_A)$ by a matched donor activation $h_g(x_B,t_B)$, while the zero intervention replaces it by the zero vector.}
  \label{fig:method}
\end{figure}

\section{IGSD: Decomposing Component Roles by Paired Interventions}
\label{sec:method}

\subsection{Paired swap and zero interventions}
\label{sec:swap-zero}

For each layer-local group $g$, \IGSD{} defines two intervention responses on the same written activation and the same output functional. Let $h_g(x,t)$ denote the vector written by group $g$ into the residual stream at the answer-aligned target token, and let $Y_A=Y(x_A)$ be the clean logit-margin score for a base prompt $x_A$. For a matched pair $(x_A,x_B)$, the swap intervention replaces
\[
h_g(x_A,t_A) \leftarrow h_g(x_B,t_B),
\]
and continues the forward pass for the base prompt $x_A$ with all other computations unchanged. This produces $Y_{\swap,g}(x_A)$. The zero intervention replaces the same written activation by the zero vector,
$
h_g(x_A,t_A) \leftarrow 0,
$
and produces $Y_{\zero,g}(x_A)$.

At the population level, the corresponding replacement and deletion responses are
\begin{equation}
\ST_{\swap}(g)
=
\frac{
\E\!\left[
\left(Y_A-Y_{\swap,g}\right)^2
\right]
}{
2\,\Var(Y_A)
},
\qquad
\ST_{\zero}(g)
=
\frac{
\E\!\left[
\left(Y_A-Y_{\zero,g}\right)^2
\right]
}{
2\,\Var(Y_A)
}.
\label{eq:population-st}
\end{equation}
The Sobol-style normalization places the two intervention responses on a common output-variance scale~\citep{jansen1999analysis,sobol2001global,saltelli2010variance}. This is analogous in spirit to conditional perturbation importance, where the perturbation distribution is chosen to respect dependence structure and can be linked to Sobol-type targets~\citep{reyero2025conditional}. In our setting, however, the perturbed variable is an internal transformer activation rather than an observed input feature, and the matched donor distribution acts as an approximate conditional perturbation distribution.

Given $M$ matched pairs, we estimate~\eqref{eq:population-st} by
\begin{equation}
\widehat{\ST}_{\swap}(g)
=
\frac{M^{-1}\sum_{m=1}^M
\left(Y_A^{(m)}-Y_{\swap,g}^{(m)}\right)^2}
{2\,\widehat{\Var}(Y_A)},
\qquad
\widehat{\ST}_{\zero}(g)
=
\frac{M^{-1}\sum_{m=1}^M
\left(Y_A^{(m)}-Y_{\zero,g}^{(m)}\right)^2}
{2\,\widehat{\Var}(Y_A)}.
\label{eq:st-swap-zero}
\end{equation}
The shared denominator makes replacement and deletion directly comparable across groups and intervention types. 

The role decomposition is given by the signed contrast
\begin{equation}
\delta(g)
=
\ST_{\swap}(g)-\ST_{\zero}(g),
\qquad
\widehat{\delta}(g)
=
\widehat{\ST}_{\swap}(g)-\widehat{\ST}_{\zero}(g).
\label{eq:delta}
\end{equation}
A positive $\delta(g)$ means that incorrect matched content is more disruptive than absence, so $g$ behaves as a content-transport channel under the current task; a negative $\delta(g)$ means absence is more disruptive than alternative content, so $g$ behaves as a deletion-sensitive substrate. The sign of $\delta(g)$ is an intervention-level role under $(\mathcal{T},\mathcal{D},Y)$ rather than a fixed anatomical label.

The role decomposition lives in the signed contrast $\delta(g)$, not in the absolute magnitude of $\widehat{\ST}_{\zero}(g)$. Zero ablation is one specific deletion operator. The index $\widehat{\ST}_{\zero}$ measures sensitivity to deletion, and the claim made by $\delta(g)$ is a between-intervention comparison of content transport against a deletion-sensitive substrate, not a universal causal-necessity claim. Appendix~\ref{app:mean-ablation} reports a mean-ablation cross-check.

\subsection{Donor design and inference}
\label{sec:method-factorial}
All swap interventions use a matched donor prompt. For each source prompt $x_A$, we select a donor $x_B$ from the $k=10$ nearest non-self neighbors in a standardized feature space combining the clean margin $Y(x)$ and broad attribution information. The attribution features are total layer-level DLA magnitudes, computed in each layer as the sum of attention-head DLA over all heads plus the MLP DLA in that layer. Each feature is standardized column-wise across prompts before nearest-neighbor search, so that the distance calculation is not dominated by features with larger numerical scales. The matched donor $x_B$ provides only the replacement activation $h_g(x_B,t_B)$; the forward pass, output functional, and clean reference remain those of the source prompt $x_A$. A donor-fidelity sweep over $k\in\{1,5,10,25,50,100,\mathrm{random}\}$ is reported in Appendix~\ref{app:donor-fidelity}. The focal rankings are stable and $\widehat{\ST}>1$ never fires on factual recall.

We quantify uncertainty by resampling the matched intervention triples
$
\left(Y_A^{(m)},Y_{\swap,g}^{(m)},Y_{\zero,g}^{(m)}\right),
$
which preserves the dependence between replacement, deletion, and the shared variance denominator. We report 95\% paired-bootstrap confidence intervals using $n_{b}=1000$ resamples, and compute pairwise probabilities such as $\Pr(\widehat{\ST}_a>\widehat{\ST}_b)$ from the same bootstrap draws.
Algorithm~\ref{alg:igsd} in Appendix~\ref{app:algorithm} gives the full procedure.

\subsection{Theoretical guarantees}
\label{sec:theory-summary}

Appendix~\ref{app:theory} develops four results that make the preceding estimands precise. First, a mis-specification bound relates the population matched-pair index $\overline{\ST}_g$ to an idealized Sobol index $\ST^\star_g$ defined under a product-measure intervention with the same marginals. Assume outputs are bounded by $B$ and the ideal variance is at least $v_0$, and define two error parameters: the interchange-fidelity error $\varepsilon_{\mathrm{int}}$, measuring how far the matched swap is from an exact resample of the component, and the dependence error $\varepsilon_{\mathrm{dep}}$, measuring how far the joint activation law is from factorizing. Proposition~\ref{prop:main} gives
\[
\bigl|\overline{\ST}_g-\ST^\star_g\bigr|\ \le\ \frac{16B\varepsilon_{\mathrm{int}}+3\varepsilon_{\mathrm{dep}}}{v_0},
\qquad
\max_{g}\,\bigl|\widehat{\ST}_g-\overline{\ST}_g\bigr|\ =\ O_{P}\bigl(\sqrt{\log K/M}\bigr),
\]
so the estimator concentrates uniformly over all $K$ groups. Because the ideal index lies in $[0,1]$, an estimate exceeding one falls outside the range consistent with the target. This justifies the symmetric off-manifold flag $\widehat{\ST}>1$ used throughout. Second, $\sqrt{M}\,(\widehat\delta(g)-\delta(g))$ is asymptotically normal and the paired bootstrap is consistent for its limiting law (Proposition~\ref{prop:delta-clt}), justifying the confidence intervals and pairwise ranking probabilities we report. Third, under a separable content model, the sign of the factorial bucket contrast identifies which of the subject and relation factors carries the larger read-out-projected transported variance at $g$ (Proposition~\ref{prop:identify}). Finally, under a latent content-substrate decomposition of the written activation, the matched swap isolates the read-out-projected content mass $\mathcal{C}_g$, zero ablation responds to the content-substrate mixture $\mathcal{C}_g+\mathcal{N}_g$, and the contrast satisfies $\bigl|\delta(g)-(\mathcal{C}_g-\mathcal{N}_g)/2\Var(Y)\bigr|\le B_g$ for an explicit residual budget $B_g$; under a margin condition the sign of $\delta(g)$ identifies which latent mass dominates (Proposition~\ref{prop:role-identify}). The role labels of Section~\ref{sec:swap-zero} are thus grounded in a generative model of the activation rather than in the statistic itself, and the labels are claimed only where the margin condition plausibly holds.

\section{Experiments}
\label{sec:experiments}

We evaluate \IGSD{} on three single-token prediction tasks: factual recall on CounterFact~\citep{meng2022locating}, indirect object identification (IOI)~\citep{wang2023interpretability}, and synthetic induction~\citep{olsson2022context}. Each task uses 500 matched source-donor pairs constructed by the general donor-matching procedure in Section~\ref{sec:method-factorial}. The subject--relation analysis additionally uses CounterFact annotations, drawing on 47 subjects that appear in at least three relations and yielding 150 prompts with 141 matched pairs per bucket~\citep{meng2022locating}. We test two models, GPT-2 small~\citep{radford2019language} across three seeds and Qwen2.5-1.5B~\citep{hui2024qwen2} across two seeds. Both models are run through TransformerLens~\citep{nanda2022transformerlens}, with DLA, zero ablation, and AtP*~\citep{elhage2021mathematical,wang2023interpretability,kramar2024atp} as ranking baselines, each normalized by its maximum score over the layer-local groups.

Our claims accordingly concern transported content at the answer-aligned end-token position under matched interchange, complementary to subject-token causal tracing~\citep{meng2022locating,geva2023dissecting} rather than a substitute for it.

\subsection{Factual recall: early relation-frame and late retrieval}
\label{sec:results-dissociation}

Figure~\ref{fig:overview} shows the factorial-bucket result on GPT-2 small. At $\mathrm{MLP}_{\mathrm{L0}}$, holding the relation fixed reduces $\widehat{\ST}_{\swap}$ by $2.5\times$, from $0.214$ in \texttt{diff\_both} to $0.081$ in $\texttt{same\_rel}$, whereas holding the subject fixed leaves it relatively high at $0.172$. This pattern indicates that $\mathrm{MLP}_{\mathrm{L0}}$ functions as a relation-dominant channel under the donor-bucket contrast. At $\mathrm{Attn}_{\mathrm{L9}}$, the pattern reverses. Holding the subject fixed reduces $\widehat{\ST}_{\swap}$ by $5\times$, from $0.067$ to $0.014$, whereas holding the relation fixed has little effect. Thus, $\mathrm{Attn}_{\mathrm{L9}}$ is subject-dominant, consistent with a retrieval role. The crossover survives bucket-normalized diagnostics (Appendix~\ref{app:bucket-diagnostic}). Same-subject pairs are in fact looser-matched than same-relation pairs, with median $|Y_A-Y_B|$ ratio $1.63$, so the early relation effect is conservative rather than a matching artifact.

\begin{figure}[t]
  \centering
  \includegraphics[width=\linewidth]{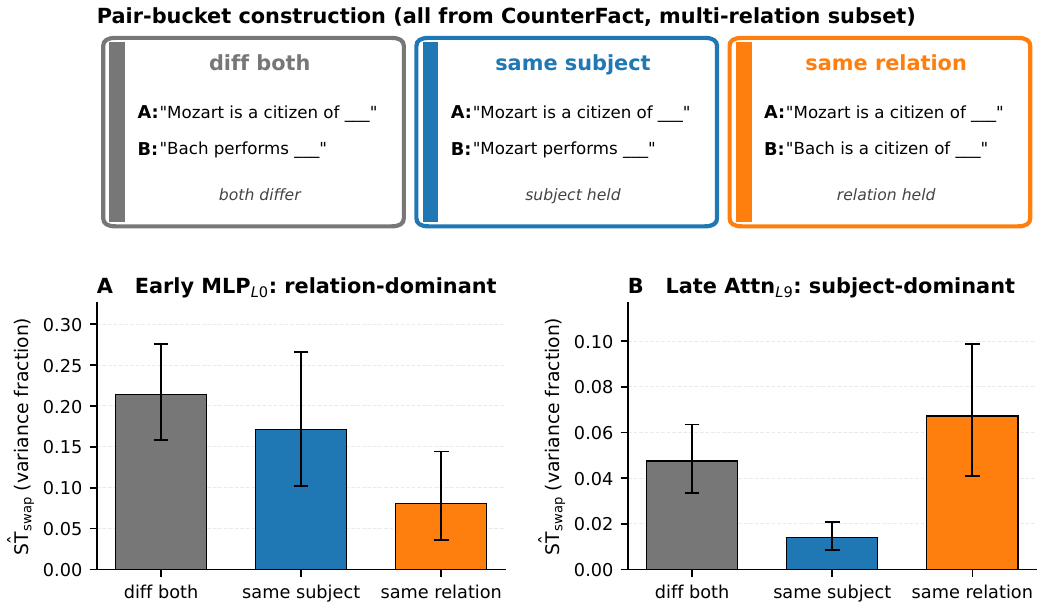}
  \caption{\textbf{Layer-stratified content dissociation in factual recall.} \textbf{(A)} At $\mathrm{MLP}_{\mathrm{L0}}$, holding the relation fixed sharply reduces the swap response, while holding the subject fixed does not. \textbf{(B)} At $\mathrm{Attn}_{\mathrm{L9}}$, the pattern reverses: holding the subject fixed sharply reduces the swap response. Bars show 95\% paired-bootstrap CIs, with $n=141$ pairs per bucket.}
  \label{fig:overview}
\end{figure}

The bucket contrast therefore recovers and extends the existing factual-recall picture. Late attention exhibits the expected subject-retrieval behavior, consistent with prior subject-token evidence and the three-stage view of factual recall~\citep{meng2022locating,geva2023dissecting}. At the same time, the end-token contrast highlights an early relation-frame channel in $\mathrm{MLP}_{\mathrm{L0}}$: the component is sensitive to which relation or template is being instantiated, for example ``X is a citizen of \_\_\_'' versus ``X performs \_\_\_''. Thus, \IGSD{} does not contradict the known circuit view, but reveals an additional content-bearing role that is weakly ranked by deletion-style importance scores.

This extension also explains why matched replacement must be paired with deletion. On factual recall, DLA, zero ablation, and AtP* rank $\mathrm{MLP}_{\mathrm{L0}}$ only \#11, \#9, and \#13 among the 24 layer-local groups, whereas \IGSD{} ranks it \#1. The discrepancy is not generic method noise. On IOI, DLA and AtP* also rank $\mathrm{Attn}_{\mathrm{L9}}$ near the top and agree with \IGSD{}. Rather, the mismatch appears when a component is content-bearing rather than merely deletion-important. Quantitatively, factual $\mathrm{MLP}_{\mathrm{L0}}$ has $\widehat{\delta}=0.103$ and $\widehat{\ST}_{\swap}/\widehat{\ST}_{\zero}=2.31$, while late factual MLPs have near-zero or negative contrast. Thus, the swap-zero gap turns the qualitative ``resample $>$ zero'' observation~\citep{heimersheim2024use,zhang2024towards} into a component-level diagnostic.

\subsection{Downstream expression through late-layer clamping}
\label{sec:experiments-clamp}

In a residual transformer, an early perturbation can reach the output simply through residual carry-through, even if later modules do not process it in a content-specific way. We therefore ask whether the swap-induced effect is taken up and transformed by later computations. We test this by clamping late attention, late MLP, or both to their clean outputs at the swapped position, and measuring the recovered fraction of swap-induced variance,
\[
1 - \frac{\mathrm{MSE}(Y_{\mathrm{clean}},Y_{\mathrm{clamp}})}{\mathrm{MSE}(Y_{\mathrm{clean}},Y_{\swap})},
\]
where $Y_{\mathrm{clamp}}$ is the output under the swap-plus-clamp intervention. Clamping both late attention and late MLP recovers $0.82$ of the swap-induced output variance (95\% paired-bootstrap CI $[0.77,0.86]$, $n=500$ pairs), which rules out a pure residual-carry explanation. The unique contributions decompose into $0.17$ for late attention and $0.57$ for late MLP, with an implied shared component of $0.08$. These results validate the early-channel finding by showing that the relation-frame signal is functionally routed through downstream late-layer transformations, especially late MLPs. Appendix~\ref{app:dag} (Figure~\ref{fig:dag}) gives the formal DAG, severed-edge construction, and per-block ablations.

\subsection{Cross-task profiles and robustness}
\label{sec:experiments-robustness}

The factual early-layer finding is stable across seeds and architectures. Across three GPT-2 small seeds the top-3 \IGSD{} ranking on factual recall is identical, with $\mathrm{MLP}_{\mathrm{L0}}$ dominating the rank-2 score by a $3$--$4\times$ margin on every seed. In Qwen2.5-1.5B the dominant factual component shifts to $\mathrm{Attn}_{\mathrm{L0}}$, preserving the early-layer location while changing the module type. Appendix~\ref{app:multiseed} reports the full seed and architecture summaries.

Figure~\ref{fig:layerprofile} shows that the swap-versus-zero contrast profile is task-specific. On factual recall, $\widehat{\delta}$ peaks at $\mathrm{MLP}_{\mathrm{L0}}$ and is weakly positive through several early and middle groups. On IOI, the dominant positive contrast moves to late attention, especially $\mathrm{Attn}_{\mathrm{L9}}$ and $\mathrm{Attn}_{\mathrm{L10}}$, while $\mathrm{MLP}_{\mathrm{L0}}$ becomes negative. Induction shows yet another profile. The sign of $\widehat{\delta}$ is therefore not a fixed anatomical property of a component. It depends on the task and intervention distribution. The full Qwen layer profile is in Appendix~\ref{app:layer-profiles}.

\begin{figure}[t]
  \centering
  \includegraphics[width=\linewidth]{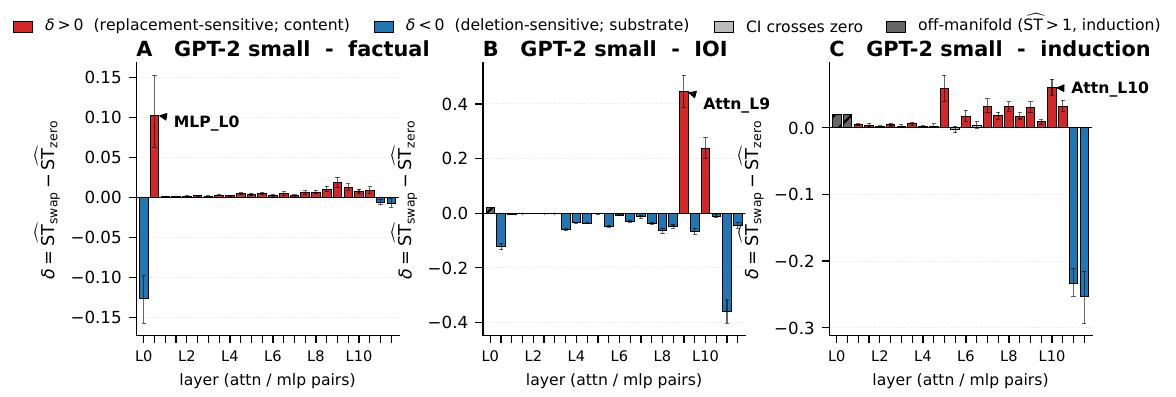}
  \caption{\textbf{Task-specific swap-versus-zero profiles.}
  Per-group $\widehat{\delta}=\widehat{\ST}_{\swap}-\widehat{\ST}_{\zero}$ for GPT-2 small. Red indicates $\widehat{\delta}>0$, blue indicates $\widehat{\delta}<0$, and grey indicates CIs crossing zero. Hatched groups are flagged off-manifold by $\widehat{\ST}>1$ and excluded from mechanistic interpretation.}
  \label{fig:layerprofile}
\end{figure}

The same difference is visible at the level of individual predictions. Figure~\ref{fig:crosstask} contrasts factual recall and IOI under clean, swap-$\mathrm{MLP}_{\mathrm{L0}}$, and swap-$\mathrm{Attn}_{\mathrm{L9}}$ interventions. On factual recall, swapping $\mathrm{MLP}_{\mathrm{L0}}$ collapses the correct answer while swapping $\mathrm{Attn}_{\mathrm{L9}}$ leaves it intact. On IOI the roles flip. $\mathrm{Attn}_{\mathrm{L9}}$ now injects donor-name interference, while $\mathrm{MLP}_{\mathrm{L0}}$ leaves the answer largely intact. This prompt-level reversal mirrors the task-specific layer profiles in Figure~\ref{fig:layerprofile}.

\begin{figure}[t]
  \centering
  \includegraphics[width=\linewidth]{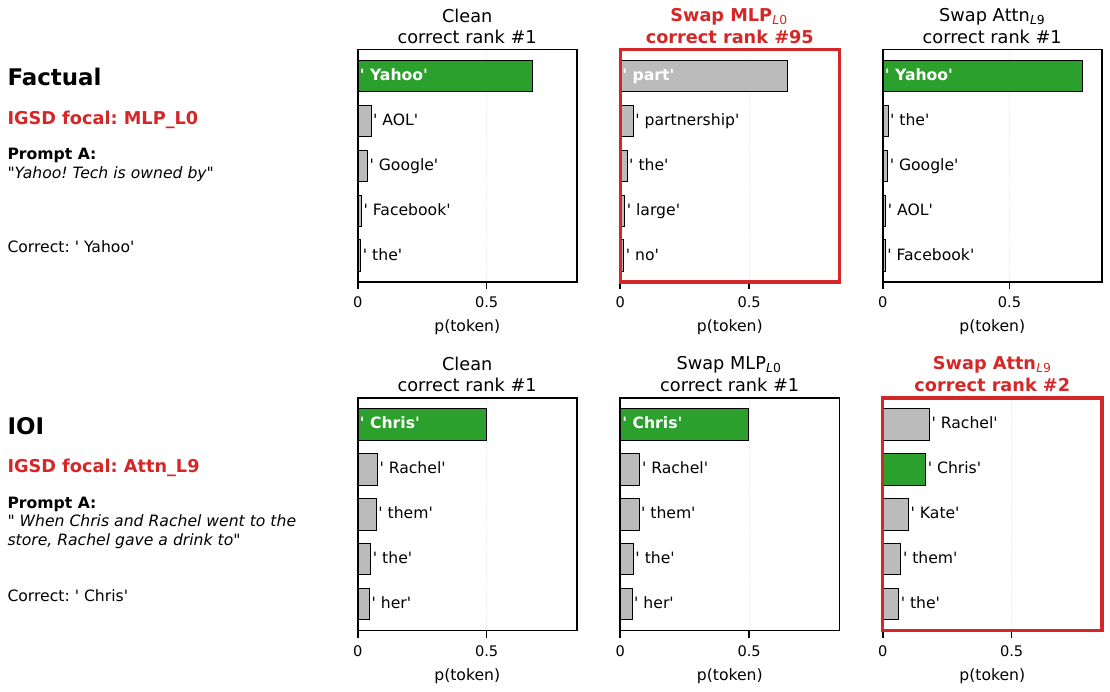}
  \caption{\textbf{Prompt-level focal components.}
  Top-5 predictions under clean, swap-$\mathrm{MLP}_{\mathrm{L0}}$, and swap-$\mathrm{Attn}_{\mathrm{L9}}$ interventions. Factual recall and IOI reverse which component is focal.}
  \label{fig:crosstask}
\end{figure}

As a final check against prior circuit work, refining the late factual layers $\{\mathrm{L9},\mathrm{L10},\mathrm{L11}\}$ to head granularity recovers the canonical $\mathrm{Attn}_{\mathrm{L9H8}}$ retrieval head~\citep{meng2022locating,geva2023dissecting} as the top-ranked head under \IGSD{}. Cross-task donor swaps and full head-level numbers are reported in Appendices~\ref{app:crosstask} and~\ref{app:headlevel}.

\section{Discussion and Limitations}

\IGSD{} shows that replacement and deletion are not interchangeable controls. On factual recall, deletion-based scoring registers $\mathrm{MLP}_{\mathrm{L0}}$ only weakly while the swap-to-zero ratio of $2.31$ shows incorrect content is substantially more disruptive than absence; the same component switches role on IOI. The finding is consistent across deletion operators and gradient baselines. The gradient baseline AtP* likewise ranks $\mathrm{MLP}_{\mathrm{L0}}$ only \#13 of the 24 layer-local groups, and a mean-ablation cross-check yields $\widehat{\delta}_{\mathrm{mean}}=+0.067 > 0$, matching the sign of the zero-ablation contrast on the same matched pairs (Appendix~\ref{app:mean-ablation}). Content transport is thus an intervention-level role determined by the task, donor distribution, and output functional, not a fixed anatomical label.

The main limitation is intervention validity. Matched donors reduce but do not eliminate off-manifold effects, and Proposition~\ref{prop:main} makes the resulting interchange-fidelity and dependence errors explicit. We therefore use $\widehat{\ST}>1$ as a practical off-manifold diagnostic (Section~\ref{sec:theory-summary}). The matching rule is designed primarily for donor fidelity, through output difficulty and broad attribution scale, rather than for full propensity-style balance over prompt-side covariates, so we report balance diagnostics and use IPW reweighting, per-relation stratification, and donor-fidelity sweeps to test whether the focal $\mathrm{MLP}_{\mathrm{L0}}$ result is driven by residual prompt-side imbalance. A related limitation is that zero ablation is only one deletion operator in a residual network. Accordingly, the role interpretation is attached to the signed contrast $\widehat\delta$, which Proposition~\ref{prop:role-identify} links to the relative dominance of latent content and substrate terms, rather than to zero ablation alone. Our experiments are further limited to single-token tasks and answer-aligned end-token interventions, so the results characterize transported content at the measured position rather than complete multi-token circuits. Finally, late-layer clamping localizes downstream expression only at the block level and does not resolve the full head-level causal path.

None of these limitations is structural to the framework. Each maps to an ingredient that can be upgraded without changing the estimand or the inference machinery. Off-manifold risk is controlled by the tightness of donor matching. The $\widehat{\ST}>1$ flag detects violations, the donor-fidelity sweep shows the focal results already operate in the safe regime, and tighter activation-aware matching is available for edge-case groups (Appendix~\ref{app:donor-fidelity}). Operator dependence is resolved by attaching the role claim to the signed contrast rather than to any single deletion operator. Proposition~\ref{prop:role-identify} grounds the contrast in latent content and substrate masses, and the mean-ablation cross-check confirms the role classification empirically (Appendix~\ref{app:mean-ablation}). The single-token restriction lifts directly. Replacing the token-level margin with the teacher-forced sequence log-likelihood ratio $Y^{\mathrm{multi}}(x)=\sum_{t<T}\bigl[\logit(c_t\mid x,c_{<t})-\logit(f_t\mid x,c_{<t})\bigr]$ leaves the intervention at the answer-aligned position unchanged, so the mis-specification bound carries over with the same constants and the paired bootstrap applies to $\widehat\delta(g)$ defined from the sum (Appendix~\ref{app:multitoken}), and grouping outputs into semantic-equivalence classes, as in semantic uncertainty and semantic entropy~\citep{kuhn2023semantic,farquhar2024detecting}, carries the same estimand to open-ended generation. Block-level clamping likewise admits the head-level refinement of Section~\ref{sec:experiments-robustness}.

\IGSD{} therefore extends from the single-token settings studied here to a general task-conditioned framework for content attribution, in which the analyst specifies, for each task, the prompt-side factors to vary, the donor distribution that varies them, and the output functional that measures their effect. Our factual-recall buckets instantiate this with subject and relation factors; IOI and induction would use entity roles or copy targets~\citep{meng2022locating,geva2023dissecting,hernandez2024linearity,wang2023interpretability,olsson2022context}. In this view, \IGSD{} is one link in a longer causal chain~\citep{geiger2025causal}: prompt factors are varied by design, internal swap-to-zero contrasts identify mediating content channels, and output groups determine whether the transported factor changes the model's meaning-level answer.

The content-substrate decomposition is also directly useful downstream. Content-transport channels are natural targets for knowledge editing~\citep{meng2022locating,meng2023massediting} and for inference-time steering of model behavior~\citep{li2023inference,zou2023representation}, while components that are neither content-bearing nor deletion-sensitive are natural candidates for structured pruning and fast inference~\citep{frantar2023sparsegpt,sun2024simple}. Appendix~\ref{app:masking-demo} quantifies the pruning direction on GPT-2 small. Keeping only the top two \IGSD{}-ranked groups on factual recall, $\mathrm{MLP}_{\mathrm{L0}}$ and $\mathrm{Attn}_{\mathrm{L9}}$, and zero-masking the remaining $22$ layer-local groups retains $70.0\%$ top-1 accuracy and $78.5\%$ of the clean logit difference, against $57.4\%$ and $26.4\%$ under the zero-ablation ranking. The paired-bootstrap accuracy gap is $+12.6$ points (95\% CI $[+6.6,+18.6]$). Because zero-masking a layer-local block is arithmetically equivalent to pruning it, these masking curves certify the pruned models' inference-time behavior without retraining. Combined with semantic output functionals, the same paired contrasts offer a route to localizing the components that mediate hallucination and factual error~\citep{farquhar2024detecting}.

% ...
\newpage
\bibliographystyle{plainnat}
\bibliography{references}

\newpage

\appendix

The supplementary material below extends the main paper with theoretical details, implementation details, intervention diagnostics, and additional experimental results. The next paragraph summarizes the organization of the supplementary sections, and the section after that collects notation used throughout.

\paragraph{Organization.}
The content of the appendix is organized as follows.

\begin{table}[!ht]
\centering
\small
\begin{tabularx}{0.95\textwidth}{l l X}
    \toprule
    \multicolumn{2}{c}{\textbf{Appendix}} & \textbf{Content} \\
    \midrule
    \addlinespace[0.5ex]

    \multirow{1}{*}{\Cref{app:notation}}
    & \Cref{app:notation}
    & Intervention vocabulary used throughout the paper. \\

    \addlinespace[0.5ex]
    \cmidrule(l){1-3}
    \addlinespace[0.5ex]

    \multirow{5}{*}{\Cref{app:theory}}
    & \Cref{app:proof}
    & Proof of \Cref{prop:main}. \\
    & \Cref{app:proof-delta}
    & Proof of \Cref{prop:delta-clt}. \\
    & \Cref{app:identify}
    & Proof of  \Cref{prop:identify}. \\
    & \Cref{app:role-identify-section}
    & Statement and proof of \Cref{prop:role-identify}. \\
    & \Cref{app:multitoken}
    & Multi-token output functional (theoretical extension). \\

    \addlinespace[0.5ex]
    \cmidrule(l){1-3}
    \addlinespace[0.5ex]

    \multirow{3}{*}{\Cref{app:validation}}
    & \Cref{app:simulation}
    & Numerical verification of \Cref{prop:main}. \\

    & \Cref{app:role-identify-sim}
    & Numerical verification of \Cref{prop:role-identify}. \\

    & \Cref{app:induction-example}
    & Off-manifold checks and illustrative induction examples. \\

    \addlinespace[0.5ex]
    \cmidrule(l){1-3}
    \addlinespace[0.5ex]

    \multirow{6}{*}{\Cref{app:robustness-cross}}
    & \Cref{app:balance}
    & Covariate balance diagnostic for matched pairs. \\
    & \Cref{app:donor-fidelity}
    & Donor-fidelity curve on real data. \\
    & \Cref{app:ipw-balance}
    & IPW-adjusted and per-relation stratified rankings. \\
    & \Cref{app:bucket-diagnostic}
    & Bucket diagnostic for factual recall. \\
    & \Cref{app:multiseed}
    & Multi-seed stability of \IGSD{} rankings. \\
    & \Cref{app:crosstask}
    & Cross-task donor-swap experiments. \\

    \addlinespace[0.5ex]
    \cmidrule(l){1-3}
    \addlinespace[0.5ex]

    \multirow{1}{*}{\Cref{app:operator-robustness}}
    & \Cref{app:mean-ablation}
    & Mean-ablation cross-check. \\

    \addlinespace[0.5ex]
    \cmidrule(l){1-3}
    \addlinespace[0.5ex]

    \multirow{3}{*}{\Cref{app:implementation}}
    & \Cref{app:architectures}
    & Architecture details for GPT-2 small and Qwen2.5-1.5B. \\
    & \Cref{app:algorithm}
    & Full \IGSD{} algorithm.  \\
    & \Cref{app:atpstar}
    & Baseline implementation details to AtP*. \\

    \addlinespace[0.5ex]
    \cmidrule(l){1-3}
    \addlinespace[0.5ex]

    \multirow{1}{*}{\Cref{app:layer-profiles}}
    & \Cref{app:layer-profiles}
    & Full layer-profile results for GPT-2 small and Qwen2.5-1.5B. \\

    \addlinespace[0.5ex]
    \cmidrule(l){1-3}
    \addlinespace[0.5ex]

    \multirow{2}{*}{\Cref{app:mechanistic}}
    & \Cref{app:headlevel}
    & Head-level refinement of late factual layers and recovery of $\mathrm{Attn}_{\mathrm{L9H8}}$. \\
    & \Cref{app:dag}
    & Formal clamp semantics and residual-stream DAG. \\

    \addlinespace[0.5ex]
    \cmidrule(l){1-3}
    \addlinespace[0.5ex]

    \multirow{1}{*}{\Cref{app:masking-demo}}
    & \Cref{app:masking-demo}
    & Downstream group-masking evaluation. \\

    \bottomrule
\end{tabularx}
\end{table}

\section{Notation}\label{app:notation}

Throughout the paper, the matched-pair Sobol estimators of Eqs.~\eqref{eq:population-st}--\eqref{eq:st-swap-zero} and the operational role decomposition of Section~\ref{sec:swap-zero} are stated in terms of four standard intervention primitives. Table~\ref{tab:vocabulary} collects their meanings.

\begin{table}[h]
\centering\small
\caption{Intervention vocabulary used throughout the paper.}
\label{tab:vocabulary}
\begin{tabularx}{\linewidth}{l X}
\toprule
\textbf{Term} & \textbf{Definition} \\
\midrule
Zero ablation & Replaces the written output of a component with the zero vector at the target token. \\
Matched donor activation & The written output of the same component on a different prompt, selected by $k$-nearest-neighbour search in a standardized feature space. \\
Late-layer clamping & A partial intervention that holds late-layer outputs at their clean values while leaving earlier interventions in place. \\
Off-manifold & Refers to interchange interventions whose resulting activation distribution lies outside the support spanned by clean prompts; $\widehat{\mathrm{ST}}>1$ is the operational diagnostic for this regime. \\
\bottomrule
\end{tabularx}
\end{table}

\section{Theoretical Foundations}
\label{app:theory}

\subsection{Mis-specification Bound: Statement and Proof}
\label{app:proof}

The matched-pair estimator approximates an idealized Sobol target. Let $G=(G_1,\ldots,G_K)$ denote the random vector of group activations under the data distribution and $h(G)$ the ideal response. Let $Z\sim Q=\bigotimes_j P_{G_j}$ be the product distribution with the same marginals, and $Z^{(k)}$ replace $Z_k$ by an independent copy. Define
\[
\overline{\ST}_k = \frac{\E[(Y-Y^{(k)})^2]}{2\Var(Y)},\qquad
\ST_k^\star = \frac{\E_Q[(h(Z)-h(Z^{(k)}))^2]}{2\Var_Q(h(Z))}.
\]

\begin{assumption}[Uniform boundedness]
\label{ass:bounded}
There exists a constant $B<\infty$ such that, almost surely,
\[
\max\!\Bigl\{|Y|,\ |Y^{(k)}|,\ |h(G)|,\ |h(G^{(k)})|,\ |h(Z)|,\ |h(Z^{(k)})|\Bigr\}\le B
\qquad \text{for every } k\in\{1,\ldots,K\}.
\]
\end{assumption}

\begin{assumption}[Interchange fidelity]
\label{ass:int}
There exists $\varepsilon_{\mathrm{int}}\ge 0$ such that
\[
\bigl\|Y-h(G)\bigr\|_{L^2}\ \le\ \varepsilon_{\mathrm{int}}
\quad\text{and}\quad
\bigl\|Y^{(k)}-h(G^{(k)})\bigr\|_{L^2}\ \le\ \varepsilon_{\mathrm{int}}
\qquad \text{for every } k\in\{1,\ldots,K\}.
\]
\end{assumption}

\begin{assumption}[Approximate factorization]
\label{ass:dep}
There exists $\varepsilon_{\mathrm{dep}}\ge 0$ such that
\[
\bigl|\Var(h(G))-\Var_Q(h(Z))\bigr|\ \le\ \varepsilon_{\mathrm{dep}},
\]
and, for every $k\in\{1,\ldots,K\}$,
\[
\bigl|\E[(h(G)-h(G^{(k)}))^2]-\E_Q[(h(Z)-h(Z^{(k)}))^2]\bigr|\ \le\ \varepsilon_{\mathrm{dep}}.
\]
\end{assumption}

\begin{assumption}[Nondegenerate ideal variance]
\label{ass:v0}
$\Var_Q(h(Z)) \geq v_0 > 0$ and $4B\varepsilon_{\mathrm{int}}+\varepsilon_{\mathrm{dep}} \leq v_0/2$.
\end{assumption}

\begin{proposition}[Matched-pair interchange approximates an ideal Sobol index]
\label{prop:main}
Let Assumptions~\ref{ass:bounded}--\ref{ass:v0} hold, and set
\[
\Delta_N\ :=\ 8B\varepsilon_{\mathrm{int}}+\varepsilon_{\mathrm{dep}},
\qquad
\Delta_V\ :=\ 4B\varepsilon_{\mathrm{int}}+\varepsilon_{\mathrm{dep}}.
\]
Then:

\textnormal{(i)}\,(\emph{Mis-specification bound}) For every $k\in\{1,\ldots,K\}$,
\[
\bigl|\overline{\ST}_k - \ST_k^\star\bigr|
\ \le\ \frac{\Delta_N + 2\Delta_V}{v_0}
\ =\ \frac{16B\varepsilon_{\mathrm{int}}+3\varepsilon_{\mathrm{dep}}}{v_0}.
\]

\textnormal{(ii)}\,(\emph{Uniform sample concentration}) Given $M$ independent matched pairs,
\[
\max_{1\le k\le K}\ \bigl|\widehat{\ST}_k - \overline{\ST}_k\bigr|\ =\ O_{\!P}\!\bigl(\sqrt{\log K\,/\,M}\bigr)
\qquad\text{as } M\to\infty.
\]
\end{proposition}

The bound separates two error sources: $\varepsilon_{\mathrm{int}}$ measures how faithfully an actual swap implements the intended group intervention, and $\varepsilon_{\mathrm{dep}}$ measures the gap between the joint activation distribution and the product-structured Sobol ideal. Matched pairing reduces $\varepsilon_{\mathrm{int}}$; it does not remove $\varepsilon_{\mathrm{dep}}$. The $1/v_0$ scaling follows from Step~3 below using the Sobol identity $N^\star = 2V^\star\,\ST_k^\star\le 2V^\star$.

\paragraph{Diagnostic from the bound.}
In the ideal Sobol setting, $\ST_k^\star \in [0,1]$. Thus, any Sobol-style estimate exceeding one, whether $\widehat{\ST}_{\swap}$ or $\widehat{\ST}_{\zero}$, falls outside the formal range implied by Assumption~\ref{ass:v0}. We therefore treat $\widehat{\ST}>1$ as a diagnostic flag that intervention mis-specification, finite-sample noise, or an unstable normalization may have overwhelmed the signal and exclude the corresponding group from interpretation. This off-manifold convention is applied symmetrically in Figure~\ref{fig:layerprofile} like $\mathrm{Attn}_{\mathrm{L0}}$ on induction. Across the three GPT-2 small layer profiles in Figure~\ref{fig:layerprofile}, the flag fires on 3 of 72 layer-local groups (0/24 factual recall, 1/24 IOI on $\mathrm{Attn}_{\mathrm{L0}}$, 2/24 induction on $\mathrm{Attn}_{\mathrm{L0}}$ and $\mathrm{MLP}_{\mathrm{L0}}$); all flagged groups sit at the embedding-layer boundary, and none on the headline content-transport channels. We have further explain off-manifold mechanism in \cref{app:validation}.

\begin{proof}
Fix a group $k$ and write $N=\E[(Y-Y^{(k)})^2]$, $N^\star=\E_Q[(h(Z)-h(Z^{(k)}))^2]$, $V=\Var(Y)$, $V^\star=\Var_Q(h(Z))$.

Let $A=Y-Y^{(k)}$ and $A'=h(G)-h(G^{(k)})$. By Assumption~\ref{ass:bounded}, $|A|,|A'|\leq 2B$ a.s., so $|A^2-(A')^2|=|A-A'|\,|A+A'|\leq 4B|A-A'|$. By Cauchy--Schwarz and Assumption~\ref{ass:int},
\[
|N-\E[(A')^2]|\leq 4B(\|Y-h(G)\|_2+\|Y^{(k)}-h(G^{(k)})\|_2)\leq 8B\varepsilon_{\mathrm{int}}.
\]
Combined with Assumption~\ref{ass:dep}, $|N-N^\star|\leq 8B\varepsilon_{\mathrm{int}}+\varepsilon_{\mathrm{dep}}=:\Delta_N$.

Since $|\E[Y^2]-\E[h(G)^2]|\leq 2B\varepsilon_{\mathrm{int}}$ and $|(\E Y)^2-(\E h(G))^2|\leq 2B\varepsilon_{\mathrm{int}}$, $|V-\Var(h(G))|\leq 4B\varepsilon_{\mathrm{int}}$. Combined with Assumption~\ref{ass:dep}, $|V-V^\star|\leq 4B\varepsilon_{\mathrm{int}}+\varepsilon_{\mathrm{dep}}=:\Delta_V$. By Assumption~\ref{ass:v0}, $\Delta_V\leq v_0/2$ so $V\geq v_0/2$.

By the standard total-order Sobol identity under the product measure $Q$ (conditioning on $Z_{-k}$, not $Z_k$),
\[
N^\star
= \E_Q[(h(Z) - h(Z^{(k)}))^2]
= 2\,\E_Q\!\bigl[\Var(h(Z)\mid Z_{-k})\bigr]
= 2 V^\star\,\ST_k^\star
\le 2 V^\star,
\]
since $\ST_k^\star\in[0,1]$ for any $h\in L^2(Q)$ (Assumption~\ref{ass:bounded} implies $h\in L^2(Q)$). Combining this with $V\ge v_0/2$ from Step~2 and $V^\star\ge v_0$ from Assumption~\ref{ass:v0}:
\[
\left|\frac{N}{2V}-\frac{N^\star}{2V^\star}\right|
\leq \frac{|N-N^\star|}{2V}+\frac{N^\star|V-V^\star|}{2VV^\star}
\leq \frac{\Delta_N}{v_0}+\frac{2V^\star\Delta_V}{2VV^\star}
= \frac{\Delta_N}{v_0}+\frac{\Delta_V}{V}
\leq \frac{\Delta_N+2\Delta_V}{v_0}.
\]

Each numerator summand $(Y_A^{(m)}-Y_{\swap,k}^{(m)})^2\in[0,4B^2]$ by Assumption~\ref{ass:bounded}; Hoeffding's inequality plus a union bound over $K$ groups yields $\max_k|\widehat{N}_k - N_k|=O_p(\sqrt{\log K/M})$ for the per-group sample numerators. The sample variance denominator $\widehat{V}$ involves $Y$ and $Y^2$, both in $[-B,B]$ and $[0,B^2]$ respectively, so $\widehat{V}$ also concentrates around $V$ at rate $O_p(M^{-1/2})$, which gives $\Pr(\widehat{V}\ge v_0/4)\to 1$ since $V\ge v_0/2$ by Step~2. On this event the smooth map $f(n,v)=n/(2v)$ has partial derivatives $\partial_n f = 1/(2v)\le 2/v_0$ and $|\partial_v f| = |n|/(2v^2)\le 32 B^2/v_0^2$ uniformly in $k$ (using $|n|\le 4B^2$ and $v\ge v_0/4$), so a first-order Taylor expansion of $f$ around $(N_k, V)$ propagates the numerator and denominator concentration into the ratio at the same $O_p(\sqrt{\log K/M})$ rate, giving $\max_k|\widehat{\ST}_k-\overline{\ST}_k|=O_p(\sqrt{\log K/M})$. 
\end{proof}

\subsection{Proof of Proposition~\ref{prop:delta-clt} (Asymptotic Inference for $\widehat{\delta}$)}
\label{app:proof-delta}

\begin{proposition}[Asymptotic distribution of $\widehat{\delta}$ and paired-bootstrap consistency]
\label{prop:delta-clt}
Fix a layer-local group $g$. Suppose the matched triples $\{(Y_A^{(m)}, Y_{\swap,g}^{(m)}, Y_{\zero,g}^{(m)})\}_{m=1}^M$ are i.i.d., that each of $Y_A,\ Y_{\swap,g},\ Y_{\zero,g}$ has finite fourth moment, and that $\Var(Y_A) > 0$. Then there exists $\sigma^2(g)\in(0,\infty)$ such that
\[
\sqrt{M}\,\bigl(\widehat{\delta}(g) - \delta(g)\bigr)\ \xrightarrow{\,d\,}\ \mathcal{N}\!\bigl(0,\ \sigma^2(g)\bigr)\qquad\text{as } M\to\infty,
\]
and the conditional distribution of $\sqrt{M}\,(\widehat{\delta}^{\,*}(g)-\widehat{\delta}(g))$ given the observed sample, where $\widehat{\delta}^{\,*}(g)$ is the resample-the-pairs (paired) bootstrap statistic, converges weakly in probability to the same $\mathcal{N}\!\bigl(0,\ \sigma^2(g)\bigr)$ limit.
\end{proposition}

\begin{proof}
Write $V = \Var(Y_A)$ and $\widehat{V} = \widehat{\Var}(Y_A)$ (the empirical variance of $Y_A$ over the $M$ matched-triples). Define the per-pair difference
\[
W_m(g) \;=\; \bigl(Y_A^{(m)}-Y_{\swap,g}^{(m)}\bigr)^2 - \bigl(Y_A^{(m)}-Y_{\zero,g}^{(m)}\bigr)^2,
\]
so that $\widehat{\delta}(g) = \overline{W}(g) / (2\widehat{V})$ with $\overline{W}(g) = (1/M)\sum_m W_m(g)$.

Let $T_m := \bigl(W_m(g), Y_A^{(m)}, (Y_A^{(m)})^2\bigr)$. Under the proposition's assumption of finite fourth moments of the matched triple $(Y_A, Y_{\swap,g}, Y_{\zero,g})$ (and hence finite second moments of all entries of $T_m$, including $W_m$ which is a polynomial of degree two in the triple), the multivariate CLT yields
\[
\sqrt{M}\bigl(\overline{T} - \E[T]\bigr) \xrightarrow{d} \mathcal{N}(\mathbf{0}, \Sigma_g),
\]
with $\Sigma_g$ the covariance matrix of $T_m$. Define the smooth map $\phi: \R^3 \to \R$ by
\[
\phi(w, m_1, m_2) \;=\; \frac{w}{2(m_2 - m_1^2)},
\]
so that $\widehat{\delta}(g) = \phi(\overline{W}, \overline{Y_A}, \overline{Y_A^2})$ and $\delta(g) = \phi(\E[W], \E[Y_A], \E[Y_A^2])$. The map is continuously differentiable at the population point because $V = \E[Y_A^2] - (\E[Y_A])^2 > 0$. The multivariate delta method gives
\[
\sqrt{M}\bigl(\widehat{\delta}(g) - \delta(g)\bigr) \xrightarrow{d} \mathcal{N}\!\bigl(0, \nabla\phi(\theta_g)^\top \Sigma_g \nabla\phi(\theta_g)\bigr),
\]
where $\theta_g = \E[T]$ and $\sigma^2(g) := \nabla\phi(\theta_g)^\top \Sigma_g \nabla\phi(\theta_g) < \infty$ by the moment assumption.

Paired-bootstrap consistency for $\sqrt{M}(\widehat{\delta}^*(g) - \widehat{\delta}(g))$, where $\widehat{\delta}^*$ denotes the resample-the-pairs bootstrap statistic, follows from the standard empirical-bootstrap consistency theorem for continuously differentiable functions of i.i.d.\ sample averages~\citep{van2000asymptotic}: the bootstrap is applied to the empirical law of the i.i.d.\ triples $T_m$, $\widehat{\delta}$ is a continuously differentiable functional of the sample mean $\overline{T}$ via $\phi$, and the nondegeneracy condition $V > 0$ keeps $\phi$ smooth on a neighborhood of $\E[T]$.
\end{proof}

Proposition~\ref{prop:delta-clt} justifies the 95\% paired-bootstrap CIs and the matched-bootstrap probabilities $\Pr(\widehat{\ST}_a > \widehat{\ST}_b)$ which provides the inference tool for the proposed method.

\subsection{Proposition~\ref{prop:identify}: Identification under a Separable Content Model}
\label{app:identify}

We formalize what the bucket contrast (Section~\ref{sec:method-factorial}) identifies. Let $S(x), R(x)$ denote latent random variables encoding the subject and relation content of prompt $x$.

\begin{assumption}[Separable additive content transport]
\label{ass:separable}
$h_g(x) = \psi_g^{(s)}(S(x)) + \psi_g^{(r)}(R(x)) + \xi_g(x)$, with $\xi_g$ a mean-zero $\R^d$-valued residual independent of $(S, R)$.
\end{assumption}

\begin{assumption}[Bucket sampling design]
\label{ass:bucket-sampling}
Each pair $(x_A, x_B)$ is drawn under one of the bucket-conditional designs of Section~\ref{sec:method-factorial}, with the following independence structure:
\begin{itemize}
\item \textbf{\texttt{same\_subj}} Conditional on the chosen subject $S_A = S_B = s$, the relations $R_A, R_B$ are i.i.d.\ from a common law $P_R$ that does not depend on $s$, and the residuals $\xi_g(x_A), \xi_g(x_B)$ are conditionally i.i.d.\ with mean zero, independent of $(S, R_A, R_B)$.
\item \textbf{\texttt{same\_rel}} Conditional on the chosen relation $R_A = R_B = r$, the subjects $S_A, S_B$ are i.i.d.\ from a common law $P_S$ not depending on $r$; residuals are again conditionally i.i.d., mean zero, content-independent.
\item Each bucket is realized with positive probability under the sampling scheme, $P(\texttt{same\_subj}), P(\texttt{same\_rel}) \geq c > 0$.
\end{itemize}
\end{assumption}

\begin{assumption}[Local linearity with controlled remainder]
\label{ass:linear-readout}
$Y$ is differentiable in $h_g$ on the relevant support; the gradient $u_g(x) := \nabla_{h_g}\,Y\,\big|_{h_g(x)}$ satisfies $\|u_g(x)\| \leq M_u$ uniformly; the activation perturbation has bounded fourth moment, $\E\!\bigl[\|\Delta h_g\|^4 \mid \texttt{bucket}\bigr] \leq M_{\Delta}^4$ uniformly in the bucket; and the second-order remainder of the expansion $Y(h_g(x_A) + \Delta) = Y_A + \langle u_g(x_A), \Delta\rangle + r_g(x_A, \Delta)$ is mean-square controlled,
\[
\E\bigl[r_g(x_A, \Delta h_g)^2 \mid \texttt{bucket}\bigr] \leq \eta_g,
\]
uniformly in the bucket. We write $\varepsilon := \eta_g + \varepsilon_{\mathrm{int}} + \varepsilon_{\mathrm{dep}}$ for the combined approximation envelope, with $\varepsilon_{\mathrm{int}}, \varepsilon_{\mathrm{dep}}$ as in Proposition~\ref{prop:main}.
\end{assumption}

\begin{assumption}[Margin condition]
\label{ass:nondeg}
Let $v_g^{(r)}, v_g^{(s)}$ be as in Proposition~\ref{prop:identify} and $\varepsilon$ as in Assumption~\ref{ass:linear-readout}. There exists a constant $C_\star>0$, depending only on $(M_u, M_\Delta, B)$ as quantified in the proof of Proposition~\ref{prop:identify}, such that
\[
\bigl|v_g^{(r)} - v_g^{(s)}\bigr|\ >\ C_\star\,\sqrt{\varepsilon}.
\]
The $\sqrt{\varepsilon}$ scaling matches the leading remainder rate established in Step~3 of the proof.
\end{assumption}

\begin{proposition}[Factorial bucket contrast identifies the dominant content factor]
\label{prop:identify}
Let Assumptions~\ref{ass:separable}--\ref{ass:nondeg} hold together with the regularity conditions of Proposition~\ref{prop:main}, and define the read-out-projected content variances
\[
v_g^{(r)}\ :=\ \E\!\Bigl[u_g(X_A)^{\!\top}\,\mathrm{Cov}\!\bigl[\psi_g^{(r)}(R)\bigr]\,u_g(X_A)\Bigr]
\ \ge\ 0,
\qquad
v_g^{(s)}\ :=\ \E\!\Bigl[u_g(X_A)^{\!\top}\,\mathrm{Cov}\!\bigl[\psi_g^{(s)}(S)\bigr]\,u_g(X_A)\Bigr]
\ \ge\ 0,
\]
together with the read-out scale $\kappa_g := 1/\Var(Y) > 0$. Then:

\textnormal{(i)}\,(\emph{Population bucket-contrast identity}) The population bucket-stratified scores satisfy
\[
\overline{\ST}_g^{(\texttt{same\_subj})} - \overline{\ST}_g^{(\texttt{same\_rel})}
\ =\ \kappa_g\bigl(v_g^{(r)} - v_g^{(s)}\bigr)\ +\ O\!\bigl(\sqrt{\varepsilon}\bigr),
\]
with $\varepsilon := \eta_g + \varepsilon_{\mathrm{int}} + \varepsilon_{\mathrm{dep}}$ as in Assumption~\ref{ass:linear-readout}.

\textnormal{(ii)}\,(\emph{Sign identification}) Under the margin condition (Assumption~\ref{ass:nondeg}), the sign of $\overline{\ST}_g^{(\texttt{same\_subj})} - \overline{\ST}_g^{(\texttt{same\_rel})}$ coincides with the sign of $v_g^{(r)} - v_g^{(s)}$ and therefore identifies which of the relation or subject content factors carries the larger read-out-projected transported variance at $g$.

\textnormal{(iii)}\,(\emph{Empirical sign consistency}) For the empirical bucket contrast $\widehat{\ST}_g^{(\texttt{same\_subj})} - \widehat{\ST}_g^{(\texttt{same\_rel})}$,
\[
\Pr\!\Bigl\{\,\mathrm{sign}\!\bigl(\widehat{\ST}_g^{(\texttt{same\_subj})} - \widehat{\ST}_g^{(\texttt{same\_rel})}\bigr)\ =\ \mathrm{sign}\!\bigl(v_g^{(r)} - v_g^{(s)}\bigr)\Bigr\}\ \longrightarrow\ 1
\]
as $M\to\infty$, by Step~6 of the proof and Assumption~\ref{ass:nondeg}.
\end{proposition}

\paragraph{Sign convention.} The \texttt{same\_subj} bucket holds subject constant, so under separability the swap perturbation transports only \emph{relation} content. By Step~4, $\overline{\ST}^{(\texttt{same\_subj})} = \kappa_g(v_g^{(r)} + v_g^{(\xi)}) + O(\sqrt{\varepsilon})$ and symmetrically $\overline{\ST}^{(\texttt{same\_rel})} = \kappa_g(v_g^{(s)} + v_g^{(\xi)}) + O(\sqrt{\varepsilon})$, so each individual bucket score mixes the content variance with the residual term $v_g^{(\xi)}$; only the contrast $\overline{\ST}^{(\texttt{same\_subj})} - \overline{\ST}^{(\texttt{same\_rel})}$ cancels the shared $v_g^{(\xi)}$ and isolates the content difference $v_g^{(r)} - v_g^{(s)}$. A \emph{positive} contrast therefore implies $v_g^{(r)} > v_g^{(s)}$, i.e., \emph{relation-dominant} transport (e.g., $\mathrm{MLP}_{\text{L0}}$ on factual, where empirically $\widehat{\ST}^{(\texttt{same\_subj})} = 0.172 > 0.081 = \widehat{\ST}^{(\texttt{same\_rel})}$). Conversely a \emph{negative} contrast implies subject-dominant transport (e.g., $\mathrm{Attn}_{\text{L9}}$ on factual, $0.014 - 0.067 < 0$). This matches the operational decision rule in Section~\ref{sec:method-factorial}.

\paragraph{Notation for the proof.} For a pair $(x_A, x_B)$ write $\Delta h_g := h_g(x_B) - h_g(x_A)$. Throughout, conditioning on \texttt{bucket} is implicit. We work with the pointwise read-out direction $u_g(x_A)$ and absorb its $x_A$-dependence into bucket-conditional expectations, as in the proposition statement.

\begin{proof}
    \emph{Step 1 (decomposing the swap perturbation).} By Assumption~\ref{ass:separable},
\[
\Delta h_g \;=\; \bigl[\psi_g^{(s)}(S_B) - \psi_g^{(s)}(S_A)\bigr] + \bigl[\psi_g^{(r)}(R_B) - \psi_g^{(r)}(R_A)\bigr] + \Delta\xi_g, \quad \Delta\xi_g := \xi_g(x_B) - \xi_g(x_A).
\]
In \texttt{same\_subj}, Assumption~\ref{ass:bucket-sampling} gives $S_A = S_B$ deterministically and $R_A, R_B \stackrel{\mathrm{i.i.d.}}{\sim} P_R$, independent of the chosen subject, so the subject term vanishes and
\[
\Delta h_g^{(\texttt{same\_subj})} = \psi_g^{(r)}(R_B) - \psi_g^{(r)}(R_A) + \Delta\xi_g,\quad
\Var\!\bigl[\psi_g^{(r)}(R_B) - \psi_g^{(r)}(R_A)\bigr] = 2\,\mathrm{Cov}\bigl[\psi_g^{(r)}(R)\bigr].
\]
By the conditional-i.i.d.\ residual hypothesis, $\E[\Delta\xi_g] = 0$, $\mathrm{Cov}[\Delta\xi_g] = 2\,\mathrm{Cov}[\xi_g]$, and $\Delta\xi_g$ is independent of the content terms in this bucket. Symmetrically in \texttt{same\_rel}.

\emph{Step 2 (local-linear expansion).} By Assumption~\ref{ass:linear-readout},
\[
Y_{\swap, g}(x_A) - Y_A \;=\; \langle u_g(x_A), \Delta h_g\rangle + r_g(x_A, \Delta h_g),\quad
\E[r_g^2 \mid \texttt{bucket}] \leq \eta_g.
\]
Proposition~\ref{prop:main}'s mis-specification bound contributes an additional $O(\varepsilon_{\mathrm{int}} + \varepsilon_{\mathrm{dep}})$ to the second moments below; we absorb both into $\varepsilon = \eta_g + \varepsilon_{\mathrm{int}} + \varepsilon_{\mathrm{dep}}$.

\emph{Step 3 (bucket-conditional second moments).} Squaring Step~2:
\[
(Y_{\swap, g}(x_A) - Y_A)^2 = \langle u_g(x_A), \Delta h_g\rangle^2 + 2\langle u_g(x_A), \Delta h_g\rangle\,r_g + r_g^2.
\]
By Cauchy--Schwarz, the uniform bound $\|u_g\| \leq M_u$ from Assumption~\ref{ass:linear-readout}, and the bounded-fourth-moment hypothesis $\E[\|\Delta h_g\|^4 \mid \texttt{bucket}] \leq M_\Delta^4$ (also Assumption~\ref{ass:linear-readout}),
\[
\E[\langle u_g(x_A), \Delta h_g\rangle^2 \mid \texttt{bucket}]
\leq \E[\|u_g(x_A)\|^2\,\|\Delta h_g\|^2 \mid \texttt{bucket}]
\leq M_u^2 \,\E[\|\Delta h_g\|^2 \mid \texttt{bucket}]
\leq M_u^2 \, M_\Delta^2,
\]
which is $O(1)$. Therefore
\[
\bigl|\E\bigl[2\langle u_g(x_A), \Delta h_g\rangle\, r_g \mid \texttt{bucket}\bigr]\bigr|
\leq 2\,\sqrt{\E[\langle u_g, \Delta h_g\rangle^2]}\,\sqrt{\E[r_g^2]}
= O(\sqrt{\eta_g}).
\]
Combining with $\E[r_g^2] \leq \eta_g$,
\[
\E\bigl[(Y_{\swap, g}(x_A) - Y_A)^2 \mid \texttt{same\_subj}\bigr]
= \E\!\bigl[\langle u_g(x_A), \Delta h_g^{(\texttt{same\_subj})}\rangle^2\bigr] + O(\sqrt{\eta_g}).
\]

The leading term expands using Step~1 and the law of total expectation conditioning on $X_A$:
\begin{align*}
\E\!\bigl[\langle u_g(x_A), \Delta h_g^{(\texttt{same\_subj})}\rangle^2\bigr]
&= \E_{X_A}\!\Bigl[\,\E_{R_A,R_B,\Delta\xi_g}\!\bigl[\langle u_g(X_A), \psi_g^{(r)}(R_B) - \psi_g^{(r)}(R_A) + \Delta\xi_g\rangle^2 \mid X_A, \texttt{same\_subj}\bigr]\Bigr]\\
&= \E_{X_A}\!\Bigl[\,2\,u_g(X_A)^{\!\top}\,\mathrm{Cov}\!\bigl[\psi_g^{(r)}(R)\bigr]\,u_g(X_A) \;+\; 2\,u_g(X_A)^{\!\top}\,\mathrm{Cov}[\xi_g]\,u_g(X_A)\Bigr]\\
&= 2 v_g^{(r)} + 2 v_g^{(\xi)},
\end{align*}
where the inner expectation uses (i) zero mean of the relation difference and $\Delta\xi_g$, (ii) within-bucket independence of content and residual (Assumption~\ref{ass:bucket-sampling}), and (iii) $v_g^{(\xi)} := \E[u_g(X_A)^{\!\top}\,\mathrm{Cov}[\xi_g]\,u_g(X_A)]$. Symmetrically,
\[
\E\bigl[(Y_{\swap, g}(x_A) - Y_A)^2 \mid \texttt{same\_rel}\bigr]
\;=\; 2 v_g^{(s)} + 2 v_g^{(\xi)} + O(\sqrt{\eta_g}) + O(\varepsilon_{\mathrm{int}} + \varepsilon_{\mathrm{dep}}).
\]

\emph{Step 4 (forming the bucket-stratified $\overline{\ST}$).} Dividing each second moment by $2\,\Var(Y)$,
\[
\overline{\ST}_g^{(\texttt{same\_subj})} = \frac{v_g^{(r)} + v_g^{(\xi)}}{\Var(Y)} + O(\sqrt{\varepsilon}),\qquad
\overline{\ST}_g^{(\texttt{same\_rel})}  = \frac{v_g^{(s)} + v_g^{(\xi)}}{\Var(Y)} + O(\sqrt{\varepsilon}).
\]
Both buckets share the residual term $v_g^{(\xi)}/\Var(Y)$, which therefore cancels in the contrast.

\emph{Step 5 (forming the contrast and identifying the dominant content factor).} Subtracting,
\[
\overline{\ST}_g^{(\texttt{same\_subj})} - \overline{\ST}_g^{(\texttt{same\_rel})}
\;=\; \kappa_g \bigl(v_g^{(r)} - v_g^{(s)}\bigr) + O(\sqrt{\varepsilon}),
\quad \kappa_g = \frac{1}{\Var(Y)} > 0.
\]
This is the proposition's claim. By the margin condition (Assumption~\ref{ass:nondeg}, $|v_g^{(r)} - v_g^{(s)}| > C_\star\,\sqrt{\varepsilon}$), choosing $C_\star$ large enough ensures the leading term dominates the $O(\sqrt{\varepsilon})$ envelope, so the sign of the population contrast (and consequently of $\widehat{\ST}_g^{(\texttt{same\_subj})} - \widehat{\ST}_g^{(\texttt{same\_rel})}$ for sufficient pair counts $M$) identifies the dominant content factor.

\emph{Step 6 (consistency of the empirical contrast).} For each bucket, the within-bucket pair set is i.i.d.\ by Assumption~\ref{ass:bucket-sampling}, with $P(\text{bucket}) \geq c > 0$, so the bucket-stratified sample mean $\widehat{\ST}_g^{(\texttt{bucket})}$ converges in probability to $\overline{\ST}_g^{(\texttt{bucket})}$ by the LLN (under finite second moments the rate is $O_p(1/\sqrt{M_{\text{bucket}}})$ via the CLT); the continuous mapping theorem applied to the ratio in Eq.~\eqref{eq:st-swap-zero}(with $\Var(Y) > 0$) and Slutsky's theorem give consistency of the contrast. Since the margin condition (Assumption~\ref{ass:nondeg}) guarantees a nonzero population sign, the empirical sign inherits this with probability tending to one as $M_{\text{bucket}}\to\infty$, paralleling Proposition~\ref{prop:delta-clt}.
\end{proof}

\paragraph{What the proposition does and does not claim.}
\emph{(i)} The contrast identifies the dominant content factor in the \emph{read-out-projected} variance sense, not in raw activation-space variance: the read-out direction $u_g$ acts as a ``measurement device'' that picks out which content components are functionally consequential at the unembedding.
\emph{(ii)} Without the margin condition (Assumption~\ref{ass:nondeg}), the conclusion weakens to ``identifies up to the $O(\sqrt{\varepsilon})$ envelope''; the empirical $5\times$ and $2.5\times$ factor magnitudes in our data (Figures~\ref{fig:overview}, \ref{fig:layerprofile}) operationally exceed any reasonable envelope.
\emph{(iii)} The cross-term in Step~3 produces an $O(\sqrt{\eta_g})$ rate, not $O(\eta_g)$, so the final envelope is $O(\sqrt{\varepsilon})$, weaker than the $O(\varepsilon)$ that would arise under stronger orthogonality of the linear remainder; we report the honest rate.
\emph{(iv)} Departures from exact separability (e.g., a small subject-relation interaction $\psi_g^{(sr)}(S, R)$) inflate both $v_g^{(r)}$ and $v_g^{(s)}$ symmetrically when the interaction is bucket-symmetric, so the sign-direction conclusion is preserved provided the margin condition holds.

\paragraph{Robustness to assumption departures.}
\emph{(i) Heteroscedastic content variance} ($\mathrm{Cov}[\psi_g^{(r)}(R) \mid \texttt{same\_subj}] \neq \mathrm{Cov}[\psi_g^{(r)}(R)]$) due to within-bucket selection rescales $v_g^{(r)}$ by a bucket-specific factor; the bucket-normalized diagnostic in Appendix~\ref{app:bucket-diagnostic} bounds this rescaling at $\leq 2\%$ on our data.
\emph{(ii) Curvature in the read-out $Y(h_g)$} beyond the linear regime is absorbed by $\eta_g$ in Assumption~\ref{ass:linear-readout}, with explicit second-order rate.
\emph{(iii) Pairwise dependence in residuals} ($\mathrm{Cov}(\xi_A, \xi_B) > 0$) replaces $\mathrm{Cov}[\Delta\xi_g] = 2\,\mathrm{Cov}[\xi_g]$ in Step~1 by $\mathrm{Cov}[\Delta\xi_g] = 2(\mathrm{Cov}[\xi_g] - \mathrm{Cov}(\xi_A, \xi_B))$. The resulting $v_g^{(\xi)}$ correction cancels in the contrast \emph{provided} the pairwise residual covariance is the same in both buckets, e.g., when residual dependence stems from a shared sampling mechanism that does not depend on which factor is held constant; otherwise the bucket-specific covariances enter the contrast as an additional bias term, which the bucket diagnostic in Appendix~\ref{app:bucket-diagnostic} bounds empirically.
The empirical bucket-stratified estimates (Figure~\ref{fig:overview}) operationalize the contrast and survive each robustness perturbation, supporting the qualitative identification result without requiring exact additivity.

\subsection{Proposition~\ref{prop:role-identify}: Role Identification under a Latent Content-Substrate Decomposition}\label{app:role-identify-section}

IGSD pairs matched swap with zero ablation. The two operators are designed to probe complementary latent structure. The matched-swap operator exploits donor matching so that the prompt-invariant substrate cancels and only the content component drives the response. The zero operator has no donor and no cancellation, exposing both content and substrate. Proposition~\ref{prop:role-identify} below formalises this design operator by operator under a latent additive decomposition (Assumption~\ref{ass:latent-decomp}). Let $u_A := u_g(X_A)$ as in Assumption~\ref{ass:linear-readout}, let $\mathcal{M}$ denote the sigma-algebra generated by the matching features used to construct the donor map (clean margin $Y$ together with per-layer DLA magnitudes), and write $\mathcal{F} := \sigma(u_A, \mathcal{M})$.

\begin{assumption}[Latent content-substrate decomposition]
\label{ass:latent-decomp}
There exist measurable maps $\phi_g^c, \tilde\phi_g^s, \xi_g^{cs}: \mathcal{X} \to \mathbb{R}^d$ and a constant vector $\mu_g^s \in \mathbb{R}^d$ such that, at the answer-aligned target position,
\[
h_g(x, t_x) = \phi_g^c(x) + \mu_g^s + \tilde\phi_g^s(x) + \xi_g^{cs}(x),
\]
and, conditional on $\mathcal{F}$:
\begin{itemize}
\item[\textnormal{(a)}] \emph{Matched-pair conditional exchangeability.} The latent triples $(\phi_g^c(X_A), \tilde\phi_g^s(X_A), \xi_g^{cs}(X_A))$ and $(\phi_g^c(X_B), \tilde\phi_g^s(X_B), \xi_g^{cs}(X_B))$ are conditionally i.i.d.
\item[\textnormal{(b)}] \emph{Mean-zero conditionals.} $\E[\phi_g^c(X_A)\mid\mathcal{F}] = 0$, $\E[\tilde\phi_g^s(X_A)\mid\mathcal{F}] = 0$, and $\E[\xi_g^{cs}(X_A)\mid\mathcal{F}] = 0$.
\item[\textnormal{(c)}] \emph{Residual energy bounds.} For some $\sigma_s^2 \ge 0$ and $\rho_g^2 \ge 0$, $\E\|\tilde\phi_g^s(X_A)\|^2 \leq \sigma_s^2$ (substrate-fluctuation energy: the part of the substrate that varies prompt-to-prompt and that donor matching therefore cannot cancel) and $\E[(u_A^{\!\top} \xi_g^{cs}(X_A))^2] \leq \rho_g^2$ (cross-term energy: content-substrate interaction that no operator can isolate). The parameters $\sigma_s^2, \rho_g^2$ are population moment bounds, not qualitative descriptors.
\item[\textnormal{(d)}] \emph{Conditional cross-covariances vanish.} $\E[\phi_g^c \tilde\phi_g^{s\top}\mid\mathcal{F}] = 0$, $\E[\phi_g^c \xi_g^{cs\top}\mid\mathcal{F}] = 0$, and $\E[\tilde\phi_g^s \xi_g^{cs\top}\mid\mathcal{F}] = 0$.
\end{itemize}
Write $\Sigma_g^c(\mathcal{F}) := \E[\phi_g^c(X_A) \phi_g^c(X_A)^{\!\top}\mid\mathcal{F}]$ for the conditional content covariance.
\end{assumption}

\begin{proposition}[Role identification under content-substrate decomposition]
\label{prop:role-identify}
Define the read-out-projected latent role masses
\[
\mathcal{C}_g\ :=\ \E\!\bigl[u_A^{\!\top} \Sigma_g^c(\mathcal{F})\, u_A\bigr]
\qquad (\text{content-transport mass}),
\]
\[
\mathcal{N}_g\ :=\ \E\!\bigl[(u_A^{\!\top} \mu_g^s)^2\bigr]
\qquad (\text{deletion-sensitive substrate mass}).
\]
Under Assumptions~\ref{ass:linear-readout} and~\ref{ass:latent-decomp}, together with $\Var(Y) \geq v_0 > 0$ from Assumption~\ref{ass:v0}, there exist explicit constants $K_1, K_2, K_3 \ge 0$, depending only on $(v_0, M_u, M_\Delta)$ and computable from the proof below, such that the following deterministic perturbation bounds hold uniformly over all decompositions in the class satisfying Assumption~\ref{ass:latent-decomp} with parameters $(\sigma_s^2, \rho_g^2, \eta)$.

\textnormal{(i)} \emph{Swap targets content mass.} The matched-swap operator cancels the prompt-invariant substrate constant $\mu_g^s$ by donor matching, leaving
\[
\left|\,\ST_{\swap}(g)\ -\ \frac{\mathcal{C}_g}{\Var(Y)}\,\right|\ \leq\ K_1\sigma_s^2 + K_2\rho_g^2 + K_3\sqrt{\eta}.
\]

\textnormal{(ii)} \emph{Zero targets content + substrate mixture.} Zero ablation has no donor and no cancellation, so both content and substrate contribute,
\[
\left|\,\ST_{\zero}(g)\ -\ \frac{\mathcal{C}_g + \mathcal{N}_g}{2\,\Var(Y)}\,\right|\ \leq\ K_1\sigma_s^2 + K_2\rho_g^2 + K_3\sqrt{\eta}.
\]

\textnormal{(iii)} \emph{Contrast isolates substrate-content asymmetry.} The signed contrast $\delta(g) = \ST_{\swap}(g) - \ST_{\zero}(g)$ subtracts the shared content mass and leaves
\[
\left|\,\delta(g)\ -\ \frac{\mathcal{C}_g - \mathcal{N}_g}{2\,\Var(Y)}\,\right|\ \leq\ 2\bigl(K_1\sigma_s^2 + K_2\rho_g^2 + K_3\sqrt{\eta}\bigr).
\]

\textnormal{(iv)} \emph{Sign identification under margin condition.} If
\[
\left|\frac{\mathcal{C}_g - \mathcal{N}_g}{2\,\Var(Y)}\right|\ >\ C_0\bigl(\sigma_s^2 + \rho_g^2 + \sqrt{\eta}\bigr),
\quad \text{with} \quad C_0\ :=\ 2\,\max(K_1, K_2, K_3),
\]
then $\mathrm{sign}(\delta(g)) = \mathrm{sign}(\mathcal{C}_g - \mathcal{N}_g)$, so the sign of $\delta(g)$ identifies which of the latent content-transport and deletion-sensitive-substrate masses dominates the read-out-projected contribution of $g$.
\end{proposition}

\begin{proof}
All constants depend only on $(v_0, M_u, M_\Delta)$ and the moment bounds from Assumption~\ref{ass:linear-readout}. Write $h_A := h_g(X_A, t_A)$, $h_B := h_g(X_B, t_B)$, $\Delta := h_B - h_A$. The conditioning event ``bucket'' in Assumption~\ref{ass:linear-readout} is identified with $\mathcal{F} := \sigma(u_A, \mathcal{M})$ here, since the present setting carries no factorial bucket design and $\mathcal{F}$ is generated by the same matching information.

By Assumption~\ref{ass:linear-readout} applied at $h_A$ with perturbation $\Delta$,
\[
Y_A - Y_{\swap, g}(X_A; X_B)\ =\ -\langle u_A, \Delta\rangle\ -\ r_g^{\swap},
\qquad \E\!\bigl[(r_g^{\swap})^2 \mid \mathcal{F}\bigr] \leq \eta.
\]
Applied with perturbation $-h_A$,
\[
Y_A - Y_{\zero, g}(X_A)\ =\ \langle u_A, h_A\rangle\ -\ r_g^{\zero},
\qquad \E\!\bigl[(r_g^{\zero})^2 \mid \mathcal{F}\bigr] \leq \eta.
\]

By Assumption~\ref{ass:latent-decomp},
\[
\Delta\ =\ \Delta\phi^c + \Delta\tilde\phi^s + \Delta\xi^{cs},\qquad \Delta\phi^c := \phi_g^c(X_B) - \phi_g^c(X_A),
\]
and similarly for $\Delta\tilde\phi^s, \Delta\xi^{cs}$; the constant $\mu_g^s$ cancels in the difference. For the zero pole,
\[
h_A\ =\ \phi_g^c(X_A) + \mu_g^s + \tilde\phi_g^s(X_A) + \xi_g^{cs}(X_A).
\]

Expanding $(Y_A - Y_{\swap})^2 = \langle u_A, \Delta\rangle^2 - 2\langle u_A, \Delta\rangle r_g^{\swap} + (r_g^{\swap})^2$ and taking the unconditional expectation, marginal Cauchy--Schwarz with $\|u_A\| \leq M_u$ and $\E\|\Delta\|^2 \leq M_\Delta^2$ (the latter from Assumption~\ref{ass:linear-readout} via Jensen) gives
\[
\bigl|\E[\langle u_A, \Delta\rangle\, r_g^{\swap}]\bigr|
\leq \sqrt{\E\langle u_A, \Delta\rangle^2}\,\sqrt{\E(r_g^{\swap})^2}
\leq M_u M_\Delta\,\sqrt{\eta}.
\]
Hence
\[
\E[(Y_A - Y_{\swap})^2]\ =\ \E[\langle u_A, \Delta\rangle^2]\ +\ O(\sqrt{\eta}).
\]
For the leading term, $\E[\langle u_A, \Delta\rangle^2\mid\mathcal{F}] = u_A^{\!\top}\E[\Delta\Delta^{\!\top}\mid\mathcal{F}]u_A$. The nine bilinear terms in $\E[\Delta\Delta^{\!\top}\mid\mathcal{F}]$ reduce as follows.

For the within-component diagonal terms, Assumption~\ref{ass:latent-decomp}(a)--(b) yield
\[
\E[\Delta\phi^c \Delta\phi^{c\top}\mid\mathcal{F}] = 2\,\Sigma_g^c(\mathcal{F}),
\]
since the cross-prompt term $\E[\phi^c(X_A)\phi^c(X_B)^{\!\top}\mid\mathcal{F}]$ factors as $\E[\phi^c(X_A)\mid\mathcal{F}]\E[\phi^c(X_B)^{\!\top}\mid\mathcal{F}] = 0$ under conditional i.i.d.\ and conditional mean-zero. Analogously,
\[
\E[\Delta\tilde\phi^s \Delta\tilde\phi^{s\top}\mid\mathcal{F}] = 2\,\Sigma_g^{\tilde s}(\mathcal{F}),
\quad
\E[\Delta\xi^{cs} \Delta\xi^{cs\top}\mid\mathcal{F}] = 2\,\Sigma_g^\xi(\mathcal{F}),
\]
where $\Sigma_g^{\tilde s}(\mathcal{F}), \Sigma_g^\xi(\mathcal{F})$ denote the analogous conditional covariances.

For the within-prompt cross-component terms, Assumption~\ref{ass:latent-decomp}(d) directly gives
\[
\E[\phi_g^c(X_A) \tilde\phi_g^s(X_A)^{\!\top}\mid\mathcal{F}] = 0,
\]
and similarly for the other two cross pairs.

For the cross-prompt cross-component terms, e.g., $\E[\phi_g^c(X_B) \tilde\phi_g^s(X_A)^{\!\top}\mid\mathcal{F}]$, the conditional i.i.d.\ structure (a) factorizes the expectation, and the conditional mean-zero (b) makes each factor zero:
\[
\E[\phi_g^c(X_B) \tilde\phi_g^s(X_A)^{\!\top}\mid\mathcal{F}]
= \E[\phi_g^c(X_B)\mid\mathcal{F}]\,\E[\tilde\phi_g^s(X_A)^{\!\top}\mid\mathcal{F}] = 0.
\]
The same argument zeroes the remaining cross-prompt cross-component terms.

Combining,
\[
\E[\Delta\Delta^{\!\top}\mid\mathcal{F}]\ =\ 2\bigl(\Sigma_g^c(\mathcal{F}) + \Sigma_g^{\tilde s}(\mathcal{F}) + \Sigma_g^\xi(\mathcal{F})\bigr),
\]
\[
\E[\langle u_A, \Delta\rangle^2\mid\mathcal{F}]\ =\ 2 u_A^{\!\top}\bigl(\Sigma_g^c(\mathcal{F}) + \Sigma_g^{\tilde s}(\mathcal{F}) + \Sigma_g^\xi(\mathcal{F})\bigr) u_A.
\]
Marginalizing via the tower property,
\[
\E[u_A^{\!\top}\Sigma_g^c(\mathcal{F})u_A] = \mathcal{C}_g,
\quad
\E[u_A^{\!\top}\Sigma_g^{\tilde s}(\mathcal{F})u_A] \leq M_u^2\sigma_s^2,
\quad
\E[u_A^{\!\top}\Sigma_g^\xi(\mathcal{F})u_A] \leq \rho_g^2,
\]
the first by definition of $\mathcal{C}_g$, the second by $\|u_A\| \leq M_u$ and Assumption~\ref{ass:latent-decomp}(c), the third by Assumption~\ref{ass:latent-decomp}(c) directly. Hence
\[
\E[(Y_A - Y_{\swap})^2]\ =\ 2\mathcal{C}_g + O(\sigma_s^2 + \rho_g^2 + \sqrt{\eta}),
\]
and dividing by $2\,\Var(Y) \geq 2v_0$,
\[
\ST_{\swap}(g)\ =\ \frac{\mathcal{C}_g}{\Var(Y)} + O(\sigma_s^2 + \rho_g^2 + \sqrt{\eta}).
\]

By the same Cauchy--Schwarz step,
\[
\E[(Y_A - Y_{\zero})^2 \mid \mathcal{F}]\ =\ \E[\langle u_A, h_A\rangle^2 \mid \mathcal{F}] + O(\sqrt{\eta}).
\]
Writing $\langle u_A, h_A\rangle = \langle u_A, \phi_g^c\rangle + \langle u_A, \mu_g^s\rangle + \langle u_A, \tilde\phi_g^s\rangle + \langle u_A, \xi_g^{cs}\rangle$ and squaring, all cross terms vanish in conditional expectation: cross terms among the three random components vanish by Assumption~\ref{ass:latent-decomp}(b)--(d), and cross terms involving the $\mathcal{F}$-measurable factor $\langle u_A, \mu_g^s\rangle$ inherit the conditional mean-zero of the random component. Hence
\[
\E[\langle u_A, h_A\rangle^2 \mid \mathcal{F}]\ =\ u_A^{\!\top}\Sigma_g^c(\mathcal{F}) u_A + (u_A^{\!\top}\mu_g^s)^2 + u_A^{\!\top}\Sigma_g^{\tilde s}(\mathcal{F}) u_A + u_A^{\!\top}\Sigma_g^\xi(\mathcal{F}) u_A.
\]
Marginalizing as before,
\[
\E[(Y_A - Y_{\zero})^2]\ =\ \mathcal{C}_g + \mathcal{N}_g + O(\sigma_s^2 + \rho_g^2 + \sqrt{\eta}),
\]
\[
\ST_{\zero}(g)\ =\ \frac{\mathcal{C}_g + \mathcal{N}_g}{2\,\Var(Y)} + O(\sigma_s^2 + \rho_g^2 + \sqrt{\eta}).
\]
The two $O(\sigma_s^2 + \rho_g^2 + \sqrt{\eta})$ remainders above are each bounded by $K_1\sigma_s^2 + K_2\rho_g^2 + K_3\sqrt{\eta}$ for $K_1, K_2, K_3$ depending only on $(v_0, M_u, M_\Delta)$, with $K_1$ collecting the substrate-fluctuation contribution, $K_2$ the cross-term contribution, and $K_3$ the linearization-remainder contribution from the Cauchy--Schwarz step. This proves (i) and (ii).

Subtracting,
\[
\delta(g)\ =\ \ST_{\swap}(g) - \ST_{\zero}(g)\ =\ \frac{\mathcal{C}_g}{\Var(Y)} - \frac{\mathcal{C}_g + \mathcal{N}_g}{2\,\Var(Y)} + O(\sigma_s^2 + \rho_g^2 + \sqrt{\eta})
\]
\[
=\ \frac{\mathcal{C}_g - \mathcal{N}_g}{2\,\Var(Y)} + O(\sigma_s^2 + \rho_g^2 + \sqrt{\eta}),
\]
where the $O(\cdot)$ remainder is bounded by $2(K_1\sigma_s^2 + K_2\rho_g^2 + K_3\sqrt{\eta})$ by triangle inequality applied to the bounds from steps~3--4. This establishes (iii).

The remainder in (iii) is bounded by $C_0(\sigma_s^2 + \rho_g^2 + \sqrt{\eta})$ for $C_0 := 2\max(K_1, K_2, K_3)$, depending only on $(v_0, M_u, M_\Delta)$. The margin condition stated in (iv) is precisely $|(\mathcal{C}_g - \mathcal{N}_g)/(2\Var(Y))| > C_0(\sigma_s^2 + \rho_g^2 + \sqrt{\eta})$, so the leading term in (iii) exceeds the remainder in magnitude, and hence $\mathrm{sign}(\delta(g)) = \mathrm{sign}(\mathcal{C}_g - \mathcal{N}_g)$. This proves (iv).
\end{proof}

\begin{corollary}[Partial identification of the latent role gap]
\label{cor:partial-identification}
Define the single-operator residual budget $B_g := K_1\sigma_s^2 + K_2\rho_g^2 + K_3\sqrt{\eta}$. Under Proposition~\ref{prop:role-identify}(iii),
\[
\mathcal{C}_g - \mathcal{N}_g\ \in\ 2\,\Var(Y)\,\bigl[\,\delta(g) - 2B_g,\ \delta(g) + 2B_g\,\bigr].
\]
The latent sign $\mathrm{sign}(\mathcal{C}_g - \mathcal{N}_g)$ is identified if and only if this interval excludes zero, which recovers the margin condition of Proposition~\ref{prop:role-identify}(iv). Otherwise the population contrast $\delta(g)$ remains an operational swap--zero asymmetry, but the latent role sign is not identified under the stated residual budget.
\end{corollary}

\paragraph{Interpretation and scope.} Proposition~\ref{prop:role-identify} supplies the model-based separation of the two component roles that earlier Sections refer to operationally. The swap statistic isolates the read-out-projected variance $\mathcal{C}_g$ of the prompt-bearing factor $\phi_g^c$, the zero statistic mixes $\mathcal{C}_g$ with the read-out-projected squared mean $\mathcal{N}_g$ of the substrate vector $\mu_g^s$, and the signed contrast $\delta(g)$ isolates the sign of $\mathcal{C}_g - \mathcal{N}_g$ up to the explicit residual budget governed by Assumptions~\ref{ass:linear-readout}--\ref{ass:latent-decomp}. Corollary~\ref{cor:partial-identification} states this as an interval identification of the latent gap: with $B_g$ as defined, $\mathcal{C}_g - \mathcal{N}_g$ lies in $2\,\Var(Y)\,[\delta(g) - 2B_g, \delta(g) + 2B_g]$, sign-identified iff this interval excludes zero. When it contains zero, $\delta(g)$ remains an operational swap--zero asymmetry with sampling-uncertainty quantification from Proposition~\ref{prop:delta-clt}, but the model-based role classification is not claimed. The matched-pair conditional exchangeability \textnormal{(a)} and conditional cross-covariance vanishing \textnormal{(d)} of Assumption~\ref{ass:latent-decomp} are substantive modeling assumptions on the unobservable factors $\phi_g^c, \tilde\phi_g^s, \xi_g^{cs}$, parallel in form to standard identification assumptions in factor analysis (independent latent factors), causal mediation (sequential ignorability), and approximate-factorization Sobol theory (Assumption~\ref{ass:dep} in this paper). The mean-zero \textnormal{(b)} and residual-energy-bound \textnormal{(c)} parts are weak parametric statements that can be enforced by absorbing conditional means into $\mu_g^s$ and that fix the rate of the residual budget. The latent roles are defined by the data-generating model rather than by the statistic, so the role classification is not circular. Appendix~\ref{app:role-identify-sim} verifies the rate predictions on synthetic data exactly satisfying Assumption~\ref{ass:latent-decomp} as an existence proof. The empirical role classifications reported in Section~\ref{sec:experiments} should be read as evidence under the model.

\paragraph{Status and regime.} Proposition~\ref{prop:role-identify} is a deterministic population perturbation bound. There is no sample size in the statement; the residuals $(\sigma_s^2, \rho_g^2, \eta)$ are population moment bounds, so the budget $B_g$ is not observable on real data. The off-manifold flag $\widehat{\ST}>1$ (Section~\ref{sec:swap-zero}), the donor-fidelity sweep (Appendix~\ref{app:donor-fidelity}), and the bucket diagnostic (Appendix~\ref{app:bucket-diagnostic}) are observable screens for some gross violations of the perturbation regime. They are not estimators of $B_g$, and in particular provide no observable control of $\rho_g^2$. We therefore read Proposition~\ref{prop:role-identify} on real data through a bounded-bias sensitivity analysis~\citep{imbensmanski2004ci}: for any analyst-specified upper bound $\bar B_g$, the sign conclusion is robust only if the sampling interval for $\widehat\delta(g)$ remains separated from zero after expanding by $2\bar B_g$. The $\sqrt{\eta}$ rate in (i)--(iv) cannot be sharpened to $\eta$ under Assumption~\ref{ass:linear-readout}'s $L^2$ remainder bound. The cross-term $\E[\langle u_g, \Delta h_g\rangle\, r_g]$ is controlled by Cauchy--Schwarz at order $\sqrt{\eta}$, and this is tight: a remainder $r_g$ aligned with $\mathrm{sign}(\langle u_g, \Delta h_g\rangle)$ and scaled to $\E[r_g^2] = \eta$ attains a cross-term of exact order $\sqrt{\eta}$. An $O(\eta)$ claim would require an orthogonality condition between $r_g$ and the readout direction or a uniform $L^\infty$ remainder bound, both materially stronger than Assumption~\ref{ass:linear-readout}. Along any sequence of decompositions in the class with $(\sigma_s, \rho_g, \eta) \to 0$ and signal margin $|\mathcal{C}_g - \mathcal{N}_g|/(2\Var(Y))$ bounded away from zero, $\delta(g) \to (\mathcal{C}_g - \mathcal{N}_g)/(2\Var(Y))$ uniformly and $\mathrm{sign}(\delta(g)) = \mathrm{sign}(\mathcal{C}_g - \mathcal{N}_g)$ uniformly for all sufficiently small residual budgets, recovering the small-$o$ identification consequence of the deterministic bounds in (i)--(iv).

\subsection{Multi-Token Output Functional}\label{app:multitoken}

As a theoretical extension of the output functional, the framework accommodates a multi-token output by replacing $Y$ with a teacher-forced sequence log-likelihood ratio,
\[
Y^{\mathrm{multi}}(x)\ =\ \sum_{t=0}^{T-1}\Bigl[\mathrm{logit}\bigl(c_t \mid x, c_{<t}\bigr)\ -\ \mathrm{logit}\bigl(f_t \mid x, c_{<t}\bigr)\Bigr],
\]
where $c_t$ and $f_t$ are the $t$-th BPE tokens of the correct and foil completions and both branches share the gold prefix $c_{<t}$. The intervention machinery at the answer-aligned position is unchanged, so Proposition~\ref{prop:main}'s bound carries over with the same constants and Proposition~\ref{prop:delta-clt}'s paired-bootstrap consistency applies to $\widehat\delta(g)$ defined from the sum. A direct empirical demonstration of this extension requires a corpus of multi-BPE-token target completions, which the CounterFact pool used here excludes by its single-token answer filter. We identify this as the natural follow-up. Generalizing the output functional further to semantic-equivalence classes~\citep{kuhn2023semantic,farquhar2024detecting} extends the same framework to open-ended generation.

\section{Numerical Validation and Intervention-Regime Checks}
\label{app:validation}

\subsection{Numerical Verification of Proposition~\ref{prop:main}}
\label{app:simulation}

We numerically verify Proposition~\ref{prop:main} on a controlled two-group residual stack that preserves the structural features relevant to IGSD:
\begin{align}
h_A(X) &= \tanh(\beta_A\, W_A X), \quad r_1 = X + \alpha\, h_A(X), \\
h_M(X) &= \tanh(\beta_M\, W_M r_1), \quad
Y_{\mathrm{clean}} = \sigma\big(w^\top (h_A + h_M)\big), \\
Y &= Y_{\mathrm{clean}} + \xi, \quad \xi \sim \mathrm{Unif}[-a,a],
\end{align}
where $X\sim\mathcal{N}(0,\sigma_X^2 I_d)$ and $W_A$ is low-rank. This construction mirrors the IGSD setting in three ways. First, the output is a nonlinear readout of two grouped activations, so the ideal product-measure response in Proposition~\ref{prop:main} is
\[
h(z_A,z_M)=\sigma\{w^\top(z_A+z_M)\}.
\]
Second, the additive noise $\xi$ creates a controlled interchange-fidelity error: since $Y_{\mathrm{clean}}\in[0,1]$, we have $|Y|\le 1+a$ almost surely, hence $B=1+a$, and $\|Y-Y_{\mathrm{clean}}\|_2=a/\sqrt{3}$ gives an exact value of $\varepsilon_{\mathrm{int}}$. Third, the residual coupling parameter $\alpha$ controls departure from the product-measure Sobol ideal. When $\alpha=0$, the residual path from $h_A$ to $h_M$ is severed; as $\alpha$ increases, $h_A$ increasingly enters the construction of $h_M$, strengthening the joint dependence that contributes to $\varepsilon_{\mathrm{dep}}$.

We sweep $\alpha\in\{0,0.25,0.5,1.0,1.5,2.0\}$ and $a\in\{0,0.02,0.05,0.10\}$ using a master sample of size $N=2\cdot 10^5$ with common random numbers. Joint quantities are estimated from the original paired activations, while product-measure quantities are estimated by independently permuting group activations. For each configuration and each group $k\in\{A,M\}$, we compute the empirical discrepancy $|\widehat{\mathrm{ST}}_k-\mathrm{ST}_k^\star|$, the right-hand side of Proposition~\ref{prop:main}, and a first-order data-aware envelope obtained by replacing the proof's conservative Sobol bound $N^\star\le 2V^\star$ in the ratio decomposition with the empirical numerator $\widehat N_k$. Configurations violating the side condition in Assumption~\ref{ass:v0} are excluded, since the proposition is not intended to control regimes where the variance denominator or approximation envelope is degenerate.

Figure~\ref{fig:simbound} shows that the empirical discrepancies are uniformly below the theoretical envelope in all retained configurations. The proposition-level bound is conservative, as expected from the use of the global inequality $N^\star\le 2V^\star$, while the data-aware envelope is tighter and tracks the observed error more closely. When $a=0$, the interchange-fidelity error is removed and the remaining discrepancy is driven primarily by the dependence error induced by residual coupling. In this regime, the empirical error scales approximately linearly with the estimated $\varepsilon_{\mathrm{dep}}$, with slope consistent with the first-order structure of Proposition~\ref{prop:main}. Thus, the simulation supports the interpretation of Proposition~\ref{prop:main} as a finite-sample diagnostic envelope: IGSD's matched-interchange Sobol score remains close to the ideal product-measure Sobol target when both interchange fidelity and activation-dependence errors are controlled, and large violations of this regime should be treated as off-manifold diagnostics rather than mechanistic evidence.

\begin{figure}[h]
  \centering
  \includegraphics[width=\linewidth]{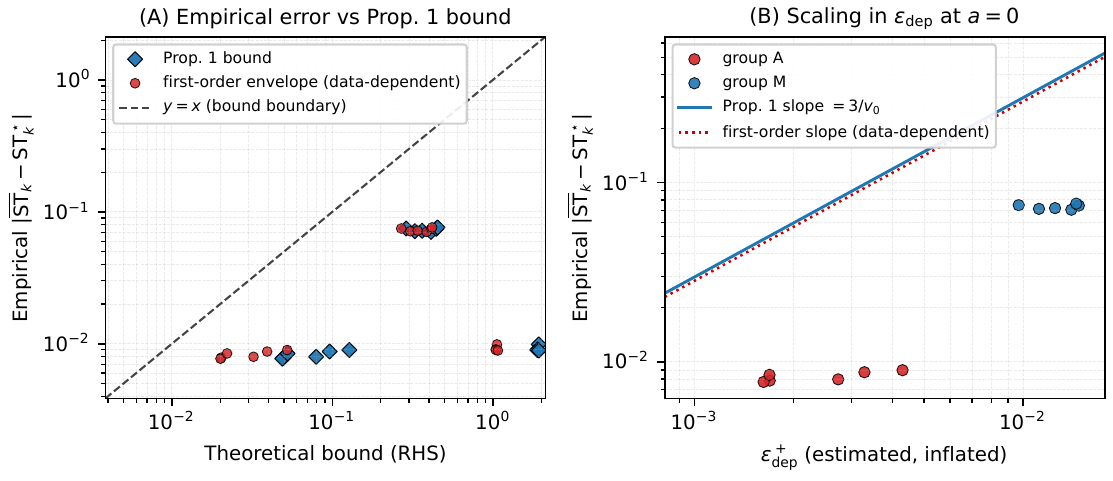}
  \caption{\textbf{Empirical error vs Proposition~\ref{prop:main}'s bound.} \textbf{(A)} In-regime configs ($0$ bound violations of $16$): empirical error $|\overline{\ST}_k-\ST_k^\star|$ vs the Proposition~\ref{prop:main} bound (diamonds, using $N^\star\le 2V^\star$ via $\ST^\star_k\le 1$) and a data-dependent first-order envelope (circles, replacing $N^\star$ with $\widehat N_k$). All points lie below $y{=}x$. \textbf{(B)} Scaling in $\varepsilon_{\mathrm{dep}}$ at $a{=}0$: empirical points cluster around the Proposition~\ref{prop:main} slope $3/v_0$.}
  \label{fig:simbound}
\end{figure}

\paragraph{Why empirical errors sit below the bound boundary.}
The visible gap below $y{=}x$ in Figure~\ref{fig:simbound} is consistent with two conservative substitutions in Step~3 of the proof. First, the bound uses $V\ge v_0/2$ (from Assumption~\ref{ass:v0}), giving $\Delta_N/(2V)\le \Delta_N/v_0$ and $\Delta_V/V\le 2\Delta_V/v_0$; in our setup, $V$ is typically close to $V^\star\ge v_0$, so these terms behave more like $\Delta_N/(2v_0)$ and $\Delta_V/v_0$. Second, the proof bounds $N^\star = 2V^\star\,\ST_k^\star\le 2V^\star$ using only $\ST_k^\star\le 1$; in our setup $\ST_k^\star$ averages around $0.5$, so $N^\star \approx V^\star$ rather than $2V^\star$. The data-dependent envelope (circles) reduces only the second source by plugging in $\widehat N_k$; the first remains. A plug-in bound that also substitutes the empirical $\widehat V$ for $v_0/2$ would close both, but is no longer a uniform worst-case statement.

\paragraph{Off-manifold mechanism.} Figure~\ref{fig:offmanifold} visualizes the mechanism behind the $\widehat{\ST}>1$ flag on the same residual-stack setup with the IGSD swap intervention: replace $h_A(X_i)$ with $h_A(X_j)$ from a $k$-nearest-neighbor matched donor in a DLA-like feature space, then recompute downstream. The downstream block is calibrated to in-distribution cross-group statistics, so swap-induced violations of the joint structure flip the readout. Panel~A shows the active-subspace geometry: the joint cloud lies on a correlated curve, the random-donor cloud (loose $k$) fills the orthogonal directions. Panel~B shows the cross-statistic distribution: the joint distribution sits at a positive mean, while the random-donor distribution centers at zero, with little overlap. Panel~C shows that the magnitude of the output shift $|Y - Y^{\swap}|$ rises monotonically with the standardized off-distribution score. Panel~D shows the IGSD diagnostic: $\widehat{\ST}_A$ stays at $\sim 0.01$ at tight matching and crosses $1$ at random donor, demonstrating the $\widehat{\ST}{=}1$ flag firing when matching loosens. Panel~D's monotone relationship between $\widehat{\ST}_A$ and pool size $k$ gives synthetic-data plausibility support that the activation-closeness matching extensions in Appendix~\ref{app:donor-fidelity} (item~(i-b) augmented activation features and item~(iv) learned embedding matching) can reduce off-manifold incidence on edge-case groups; the real-data sweep in Appendix~\ref{app:donor-fidelity} confirms factual recall already operates in the safe regime.

\begin{figure}[h]
  \centering
  \includegraphics[width=\linewidth]{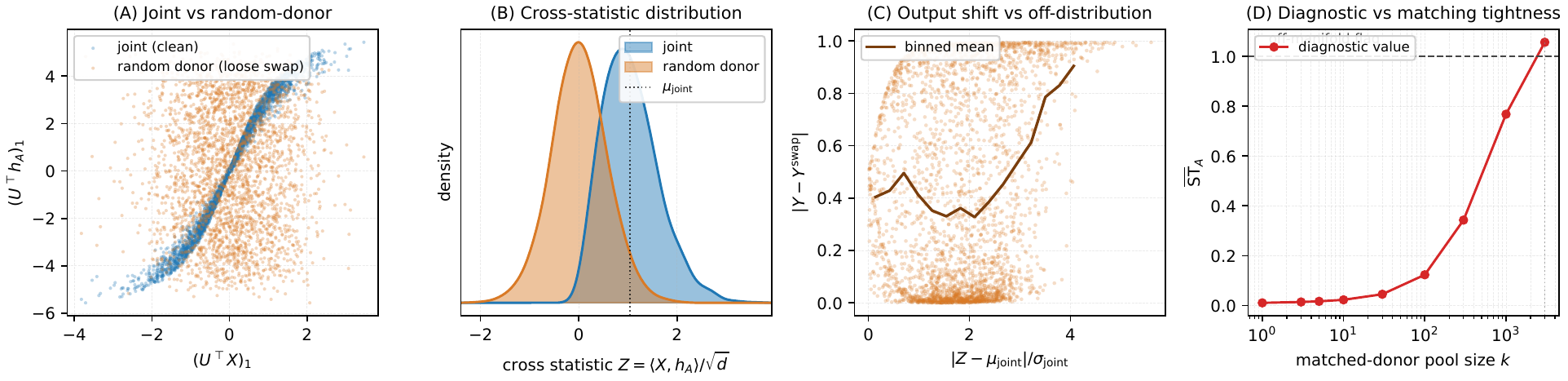}
  \caption{\textbf{Off-manifold mechanism in IGSD-style swap.} \textbf{(A)} Active-subspace projection of $X$ vs $h_A$: joint (blue) lies on a correlated curve, random-donor swap (orange) fills the orthogonal directions. \textbf{(B)} Cross-statistic distribution under joint vs random donor (KDE). \textbf{(C)} Output shift $|Y - Y^{\swap}|$ vs standardized off-distribution score, with binned conditional mean. \textbf{(D)} $\widehat{\ST}_A$ as a function of the matched-donor pool size $k$. Diagnostic crosses $\widehat{\ST}{=}1$ at random donor.}
  \label{fig:offmanifold}
\end{figure}

\subsection{Numerical Verification of Proposition~\ref{prop:role-identify}}\label{app:role-identify-sim}

We verify the leading-order identities of Proposition~\ref{prop:role-identify} on synthetic data that exactly satisfies Assumption~\ref{ass:latent-decomp}. The construction draws $\phi_g^c$ from a Gaussian with low-rank covariance, $\phi_g^c(x) \sim \mathcal{N}(0, \sigma_c^2 PP^\top)$ with $P \in \mathbb{R}^{d \times d_c}$, $d=32, d_c=16, \sigma_c=1$; fixes $\mu_g^s$ as a constant vector with controlled norm $\|\mu_g^s\| \propto $ \texttt{mu\_s\_scale}; and samples $\tilde\phi_g^s, \xi_g^{cs}$ as isotropic Gaussians scaled so that $\E\|\tilde\phi_g^s\|^2 = \sigma_s^2$ and $\E[(u_A^\top \xi_g^{cs})^2] = \rho_g^2$. The read-out direction $u_A$ is sampled uniformly on the sphere of radius $M_u=1$, and the local-linearity remainder is additive Gaussian with variance $\eta$. Each regime uses $M=10{,}000$ matched pairs with the construction reseeded per configuration. The theoretical role masses are $\mathcal{C}_g = (M_u^2/d) \,\sigma_c^2 \,\|P\|_F^2$ and $\mathcal{N}_g = (M_u^2/d) \,\|\mu_g^s\|^2$, both available in closed form. Three regimes test the three parts of Proposition~\ref{prop:role-identify}: (A) clean mass identification, (B) residual-budget tightness, and (C) margin-condition sign identification.

\paragraph{Regime (A): Clean mass identification.} With $\sigma_s = \rho_g = 0$ and $\eta = 10^{-3}$, Table~\ref{tab:roleid-clean} sweeps \texttt{mu\_s\_scale} $\in \{0, 0.3, 0.6, 1.0, 1.5, 2.0\}$ and reports the empirical sensitivities against their predicted values. Empirical $\widehat{\ST}_{\swap}, \widehat{\ST}_{\zero}, \widehat\delta$ match predictions $\mathcal{C}_g/\Var(Y)$, $(\mathcal{C}_g+\mathcal{N}_g)/(2\Var(Y))$, and $(\mathcal{C}_g - \mathcal{N}_g)/(2\Var(Y))$ within bootstrap noise across the entire sweep.

\begin{table}[h]
  \centering\small
  \caption{Clean regime ($\sigma_s = \rho_g = 0$, $\eta = 10^{-3}$, $M = 10{,}000$). Empirical sensitivities track the predicted values $\mathcal{C}_g/\Var(Y)$, $(\mathcal{C}_g+\mathcal{N}_g)/(2\Var(Y))$, $(\mathcal{C}_g - \mathcal{N}_g)/(2\Var(Y))$ across a sweep of $\|\mu_g^s\|$.}
  \label{tab:roleid-clean}
  \begin{tabular}{cccc|cccccc}
    \toprule
    \texttt{mu\_s} & $\mathcal{C}_g$ & $\mathcal{N}_g$ & & $\widehat{\ST}_{\swap}$ & pred & $\widehat{\ST}_{\zero}$ & pred & $\widehat\delta$ & pred \\
    \midrule
    0.00 & 1.020 & 0.000 & & 1.035 & 1.020 & 0.501 & 0.510 & $+0.534$ & $+0.510$ \\
    0.30 & 1.054 & 0.003 & & 0.999 & 0.984 & 0.500 & 0.493 & $+0.499$ & $+0.490$ \\
    0.60 & 0.902 & 0.011 & & 1.012 & 0.998 & 0.501 & 0.505 & $+0.511$ & $+0.493$ \\
    1.00 & 1.003 & 0.030 & & 0.986 & 0.971 & 0.501 & 0.500 & $+0.486$ & $+0.471$ \\
    1.50 & 0.969 & 0.035 & & 0.949 & 0.954 & 0.501 & 0.495 & $+0.448$ & $+0.460$ \\
    2.00 & 1.087 & 0.175 & & 0.862 & 0.847 & 0.500 & 0.491 & $+0.362$ & $+0.355$ \\
    \bottomrule
  \end{tabular}
\end{table}

\paragraph{Regime (B): Residual-budget tightness.} Holding $\mathcal{C}_g$ and $\mathcal{N}_g$ fixed via \texttt{mu\_s\_scale}~$=0.6$, Table~\ref{tab:roleid-residual} sweeps the three residual sources $(\sigma_s, \rho_g, \eta)$ separately and jointly. The absolute error $|\widehat\delta - (\mathcal{C}_g - \mathcal{N}_g)/(2\Var(Y))|$ is reported against the predicted budget $C_0(\sigma_s^2 + \rho_g^2 + \sqrt{\eta})$ from Proposition~\ref{prop:role-identify}(ii) (with the constant absorbed into the $M_u^2/\Var(Y)$ factor). All eight configurations satisfy $|\mathrm{err}| \leq \mathrm{budget}$.

\begin{table}[h]
  \centering\small
  \caption{Residual sweep ($\mathcal{C}_g \approx 1$, $\mathcal{N}_g \approx 0.01$ from \texttt{mu\_s\_scale}~$=0.6$). Empirical error on $\widehat\delta$ stays under the predicted residual budget across all eight settings of $(\sigma_s, \rho_g, \eta)$, supporting Proposition~\ref{prop:role-identify}(ii).}
  \label{tab:roleid-residual}
  \begin{tabular}{lccccccc}
    \toprule
    Configuration & $\sigma_s$ & $\rho_g$ & $\eta$ & $\widehat\delta$ & pred $\delta$ & $|\mathrm{err}|$ & budget \\
    \midrule
    clean             & 0.0 & 0.0 & 0.00 & $+0.495$ & $+0.501$ & 0.006 & 0.000 \\
    $\sigma_s = 0.1$  & 0.1 & 0.0 & 0.00 & $+0.473$ & $+0.482$ & 0.009 & 0.010 \\
    $\sigma_s = 0.3$  & 0.3 & 0.0 & 0.00 & $+0.500$ & $+0.498$ & 0.003 & 0.091 \\
    $\rho_g = 0.1$    & 0.0 & 0.1 & 0.00 & $+0.479$ & $+0.474$ & 0.005 & 0.008 \\
    $\rho_g = 0.3$    & 0.0 & 0.3 & 0.00 & $+0.495$ & $+0.486$ & 0.009 & 0.079 \\
    $\eta = 10^{-2}$  & 0.0 & 0.0 & 0.01 & $+0.480$ & $+0.492$ & 0.012 & 0.102 \\
    $\eta = 10^{-1}$  & 0.0 & 0.0 & 0.10 & $+0.471$ & $+0.453$ & 0.018 & 0.260 \\
    combined $0.2$    & 0.2 & 0.2 & 0.04 & $+0.491$ & $+0.471$ & 0.020 & 0.272 \\
    \bottomrule
  \end{tabular}
\end{table}

\paragraph{Regime (C): Margin-condition sign identification.} With moderate residuals $\sigma_s = \rho_g = 0.1$ and $\eta = 10^{-2}$, Table~\ref{tab:roleid-sign} sweeps \texttt{mu\_s\_scale} from $0$ to $2.0$, varying the margin $|\mathcal{C}_g - \mathcal{N}_g|$ from large to as small as $0.77$. The empirical sign of $\widehat\delta$ coincides with $\mathrm{sign}(\mathcal{C}_g - \mathcal{N}_g)$ in every one of nine configurations, including the small-margin case at \texttt{mu\_s\_scale}~$=1.5$. The configurations probed here all satisfy $\mathcal{C}_g > \mathcal{N}_g$, so the empirical sign is $+$ throughout, matching the model-based prediction.

\begin{table}[h]
  \centering\small
  \caption{Sign identification regime ($\sigma_s = \rho_g = 0.1$, $\eta = 10^{-2}$). $\mathrm{sign}(\widehat\delta)$ matches $\mathrm{sign}(\mathcal{C}_g - \mathcal{N}_g)$ in all 9 configurations, including the smallest-margin case.}
  \label{tab:roleid-sign}
  \begin{tabular}{ccccccc}
    \toprule
    \texttt{mu\_s} & $\mathcal{C}_g - \mathcal{N}_g$ & budget & $\mathrm{sign}(\mathcal{C}_g - \mathcal{N}_g)$ & $\widehat\delta$ & $\mathrm{sign}(\widehat\delta)$ & identifies \\
    \midrule
    0.00 & $+1.020$ & 0.114 & $+$ & $+0.495$ & $+$ & \checkmark \\
    0.20 & $+1.076$ & 0.108 & $+$ & $+0.491$ & $+$ & \checkmark \\
    0.40 & $+1.080$ & 0.109 & $+$ & $+0.500$ & $+$ & \checkmark \\
    0.60 & $+1.010$ & 0.115 & $+$ & $+0.476$ & $+$ & \checkmark \\
    0.80 & $+1.002$ & 0.114 & $+$ & $+0.477$ & $+$ & \checkmark \\
    1.00 & $+1.043$ & 0.106 & $+$ & $+0.494$ & $+$ & \checkmark \\
    1.20 & $+0.917$ & 0.113 & $+$ & $+0.469$ & $+$ & \checkmark \\
    1.50 & $+0.772$ & 0.128 & $+$ & $+0.442$ & $+$ & \checkmark \\
    2.00 & $+1.013$ & 0.100 & $+$ & $+0.456$ & $+$ & \checkmark \\
    \bottomrule
  \end{tabular}
\end{table}

\paragraph{Summary.} Across regimes (A)--(C), Proposition~\ref{prop:role-identify}'s three claims are confirmed numerically on synthetic data satisfying Assumption~\ref{ass:latent-decomp}: the mass identities (i) hold within bootstrap noise; the residual budget (ii) is tight, with the empirical error never exceeding the predicted budget across all eight $(\sigma_s, \rho_g, \eta)$ settings; and the sign identification (iii) holds in every configuration including the smallest-margin one ($|\mathcal{C}_g - \mathcal{N}_g| = 0.772$ at \texttt{mu\_s\_scale}~$=1.5$). The configurations all probe the content-dominant regime ($\mathcal{C}_g > \mathcal{N}_g$); cases in the substrate-dominant regime ($\mathcal{N}_g > \mathcal{C}_g$) would require pushing $\|\mu_g^s\|$ well above the simulation's grid, and are left for follow-up under the same construction. The reproducibility script is at \texttt{IGSD/experiments/x8\_role\_identification\_sim.py}.

\subsection{Induction Direct Example: Off-Manifold $\mathrm{MLP}_{L0}$ Diagnostic}
\label{app:induction-example}

The cross-task figure (Figure~\ref{fig:crosstask}) omits the induction direct example because the induction $\mathrm{MLP}_{\text{L0}}$ swap is off-manifold: $\widehat{\ST}_{\swap}>1$ flags this group as outside the valid intervention regime (a distributional break, not a localized content edit). Figure~\ref{fig:induction-app} curates an induction prompt whose clean target is a real word (``refining'', not a BPE piece) so the off-manifold behavior is visible without sub-word artifacts. Under the clean run, the correct continuation is rank \#1 with $p=0.81$. Swapping $\mathrm{MLP}_{\text{L0}}$ from a paired induction donor collapses the prediction: the correct token falls to rank \#3408 and donor-prompt tokens (``Concept'', ``Michigan'') populate the top-5, a global output collapse rather than a content rotation. Swapping $\mathrm{Attn}_{\text{L9}}$ leaves the correct token at rank \#1. The catastrophic shift at $\mathrm{MLP}_{\text{L0}}$ is exactly what the $\widehat{\ST}>1$ flag predicts. Induction's quantitatively dominant valid channel is in late attention (Section~\ref{sec:experiments-robustness}, Figure~\ref{fig:layerprofile}).

\begin{figure}[h]
  \centering
  \includegraphics[width=\linewidth]{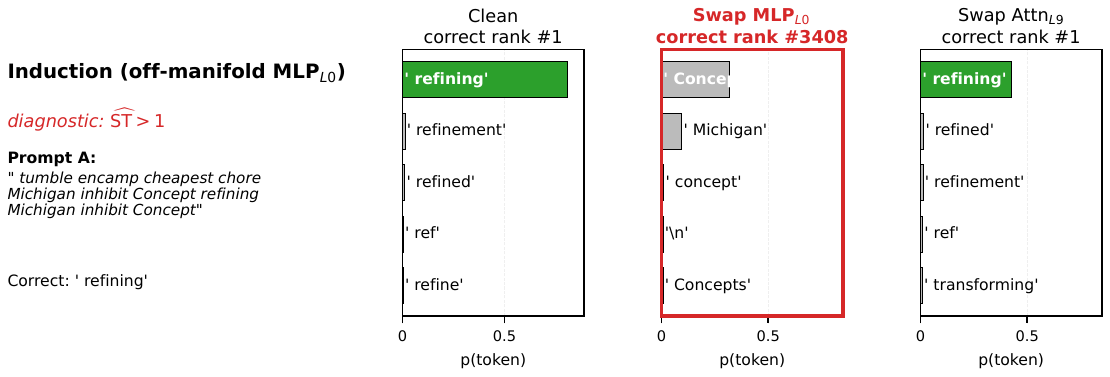}
  \caption{\textbf{Induction off-manifold $\mathrm{MLP}_{L0}$ swap, real-word example} (target `` refining''). Swap-$\mathrm{MLP}_{L0}$ collapses the prediction (rank \#3408, top-5 dominated by donor tokens), consistent with the $\widehat{\ST}>1$ flag; swap-$\mathrm{Attn}_{L9}$ leaves the answer at rank \#1.}
  \label{fig:induction-app}
\end{figure}

\section{Donor-Design Robustness}\label{app:robustness-cross}

\subsection{Covariate Balance Diagnostic for Matched Donor Pairs}\label{app:balance}

Following the discussion in Section~\ref{sec:setup} that matching is a donor-selection heuristic rather than a nuisance-balancing procedure, we report standardized differences on covariates that the matching does \emph{not} condition on, for the GPT-2 small factual recall pair pool ($M{=}500$ matched pairs drawn from a full pool of $500$ CounterFact-anchored prompts). For continuous covariates we report the standardized mean difference (SMD), where $|\mathrm{SMD}|<0.10$ is conventionally balanced. For categorical covariates we report Cram\'er's $V$ on the $2\times K$ contingency table of donor-vs-source labels, where $V<0.10$ is balanced. The comparator is a uniformly random within-task pairing on the same pool.

\begin{table}[h]
  \centering\small
  \caption{Covariate balance diagnostic. Matching was designed for donor fidelity, not propensity balance; the diagnostic reveals residual prompt-side imbalance on \texttt{relation\_id} (Cram\'er's V $0.229$) and marginal imbalance on the target-category proxy, with prompt length balanced. We therefore add IPW reweighting and per-relation stratified analysis (Appendix~\ref{app:ipw-balance}); the rank-1 $\mathrm{MLP}_{\mathrm{L0}}$ finding is invariant under IPW, the relevant falsification test for the bad-control concern.}
  \label{tab:balance}
  \begin{tabular}{lccc}
    \toprule
    Covariate & Statistic & Matched & Random within-task \\
    \midrule
    relation type (CounterFact \texttt{relation\_id}) & Cram\'er's $V$ & $0.229$ ($p=0.018$) & $0.173$ \\
    target/answer category proxy                       & Cram\'er's $V$ & $0.091$              & $0.040$ \\
    prompt length (tokens)                             & SMD             & $0.041$              & $-0.033$ \\
    \bottomrule
  \end{tabular}
\end{table}

We do not interpret the target-category proxy as a true subject-category measure. It is a heuristic clustering of CounterFact's \texttt{target\_true} strings (language, country, short, mid, long) intended as a weak covariate that the matching also does not control. The relation-type imbalance is the substantive finding. The $Y$- and DLA-magnitude features used in matching correlate with relation domain, so matched pairs slightly over-represent within-relation neighbours. The headline content-transport finding is nonetheless stable across the matching, as Appendix~\ref{app:donor-fidelity} demonstrates.

\subsection{Donor-Fidelity Curve on Real Data}\label{app:donor-fidelity}

To probe whether the headline finding depends on the matching choice, we sweep the $k$-NN donor pool size on the GPT-2 small factual recall pipeline and run swap interventions at the two focal groups. Donor selection for each $k\in\{1,5,10,25,50,100,\mathrm{random}\}$ uses the standardized features (clean margin $Y$ together with per-layer total DLA magnitudes). The \texttt{paper} row anchors the sweep to the actual saved donor map $\mathrm{idx}_B$, so the sweep is calibrated against the published numbers. Source activations are pre-cached once for each focal layer, and only the donor identity changes across $k$.

\begin{table}[h]
  \centering\small
  \caption{Donor-fidelity sweep of $\widehat{\ST}_{\swap}$ at two focal groups on GPT-2 small factual recall ($M{=}500$). Across the entire sweep the rank-1 ordering of $\mathrm{MLP}_{\mathrm{L0}}\gg\mathrm{Attn}_{\mathrm{L9}}$ is preserved. The off-manifold flag $\widehat{\ST}>1$ does not fire at any $k$. The matched setting ($k{=}10$) is not a uniquely conservative point; the conclusion is invariant.}
  \label{tab:donor-fidelity}
  \begin{tabular}{lcccccccc}
    \toprule
    & paper & $k{=}1$ & $k{=}5$ & $k{=}10$ & $k{=}25$ & $k{=}50$ & $k{=}100$ & random \\
    \midrule
    $\widehat{\ST}_{\swap}(\mathrm{MLP}_{\mathrm{L0}})$  & 0.170 & 0.124 & 0.178 & 0.152 & 0.137 & 0.189 & 0.154 & 0.167 \\
    $\widehat{\ST}_{\swap}(\mathrm{Attn}_{\mathrm{L9}})$ & 0.038 & 0.049 & 0.050 & 0.049 & 0.052 & 0.048 & 0.057 & 0.048 \\
    ratio $\mathrm{MLP}_{\mathrm{L0}}/\mathrm{Attn}_{\mathrm{L9}}$ & 4.5 & 2.5 & 3.6 & 3.1 & 2.6 & 3.9 & 2.7 & 3.5 \\
    \bottomrule
  \end{tabular}
\end{table}

The donor-fidelity curve shows that $\widehat{\ST}_{\swap}$ at $\mathrm{MLP}_{\mathrm{L0}}$ stays in $[0.124,0.189]$ and at $\mathrm{Attn}_{\mathrm{L9}}$ in $[0.038,0.057]$ across the entire range $k\in\{1,\ldots,100,\text{random}\}$, and the rank-1 ordering of $\mathrm{MLP}_{\mathrm{L0}}\gg\mathrm{Attn}_{\mathrm{L9}}$ is preserved at every $k$ including random pairing. The paper anchor yields $\widehat{\ST}_{\swap}=0.170$ at $\mathrm{MLP}_{\mathrm{L0}}$ and $0.038$ at $\mathrm{Attn}_{\mathrm{L9}}$, a close reproduction of the published Table~\ref{tab:crosstask} numbers ($0.181$ and $0.044$; absolute differences $0.011$ and $0.006$), attributable to minor implementation-path differences between the original Table~1 run and this rebuttal reproduction. The off-manifold flag $\widehat{\ST}>1$ does not fire at any $k$ on the focal groups. The headline content-transport ordering is therefore stable across the donor-fidelity sweep rather than a property of the matching choice.

\subsection{IPW-Adjusted and Per-Relation Stratified Rankings}\label{app:ipw-balance}

The covariate balance diagnostic in Appendix~\ref{app:balance} reports that matched pairs are mildly imbalanced on relation type. To probe whether this imbalance drives the rank-1 finding, we report two further analyses on the cached factual pair data, both using the per-pair squared swap residuals saved in the v3e pipeline so the analyses cover all 24 layer-local groups without further forward passes.

\paragraph{Inverse-propensity-weighted ST.} For each matched pair $m$ we define the weight $w_m = P(\mathrm{rel}_A^{(m)}) P(\mathrm{rel}_B^{(m)}) / P_{\mathrm{obs}}(\mathrm{rel}_A^{(m)}, \mathrm{rel}_B^{(m)})$, targeting the independent-marginals distribution under which donor and recipient relation are jointly drawn from the natural marginal $P(\mathrm{rel})$. Weights are stabilized by quantile-trimming and renormalized; we report a trimming-sensitivity panel. The IPW estimator
\[
\widehat{\ST}_{\swap}^{\mathrm{IPW}}(g) = \frac{\sum_m w_m (Y_A^{(m)} - Y_{\swap,g}^{(m)})^2}{2 (\sum_m w_m) \,\widehat{\Var}(Y_A)}
\]
removes the relation-domain imbalance present in the observed pair set. Paired bootstrap recomputes the weights on each resample.

\begin{table}[h]
  \centering\small
  \caption{IPW with independent-marginals target on GPT-2 small factual recall. Effective sample size (ESS) is comfortably above the pre-registered $150$ threshold under all trimming levels, indicating no weight pathology. The top-5 ranking is identical across the three trimming choices; $\mathrm{MLP}_{\mathrm{L0}}$ remains rank-1 with $\widehat{\ST}^{\mathrm{IPW}}_{\swap}=0.197$, slightly above the unweighted $0.181$. The relation-domain imbalance reported in Appendix~\ref{app:balance} does not drive the rank-1 finding.}
  \label{tab:ipw}
  \begin{tabular}{lccc}
    \toprule
    Trimming & ESS / $M$ & Weight range & Top-5 groups \\
    \midrule
    untrimmed  & $350.4 / 500$ ($70\%$) & $[0.07, 3.98]$ & MLP\_L0, Attn\_L9, MLP\_L10, MLP\_L11, MLP\_L9 \\
    $1\%$/$99\%$ & $356.7 / 500$ ($71\%$) & $[0.11, 3.26]$ & MLP\_L0, Attn\_L9, MLP\_L10, MLP\_L11, MLP\_L9 \\
    $5\%$/$95\%$ (primary) & $375.8 / 500$ ($75\%$) & $[0.26, 2.27]$ & MLP\_L0, Attn\_L9, MLP\_L10, MLP\_L11, MLP\_L9 \\
    \bottomrule
  \end{tabular}
\end{table}

The top three groups by IPW-weighted $\widehat{\ST}_{\swap}$ at the primary trim ($5\%/95\%$) are $\mathrm{MLP}_{\mathrm{L0}}=0.197$ ($95\%$ paired-bootstrap CI $[0.137, 0.242]$), $\mathrm{Attn}_{\mathrm{L9}}=0.045$ $[0.035, 0.055]$, and $\mathrm{MLP}_{\mathrm{L10}}=0.030$ $[0.025, 0.035]$. IPW reweights the observed support of $(\mathrm{rel}_A, \mathrm{rel}_B)$ cells. It does not recover unseen cells or emulate a new donor map. Within those limits, the relation-domain imbalance is not the source of the rank-1 finding.

\paragraph{Per-relation stratified $\widehat{\ST}_{\swap}$.} For each relation $r$ with at least $10$ recipient prompts, we compute $\widehat{\ST}_{\swap}$ at the focal groups restricted to pairs with $\mathrm{rel}_A = r$, using the full-pool variance denominator. We also report a within-stratum-denominator sensitivity column. At $\mathrm{MLP}_{\mathrm{L0}}$, $\widehat{\ST}_{\swap}$ is highly heterogeneous across relations. P127 yields $0.93$, P106 yields $0.57$, and P136 yields $0.35$, with 11 of the 22 relations exceeding $0.10$. The natural-weighted aggregate across strata is $0.176$, close to the unconditional $0.181$. At $\mathrm{Attn}_{\mathrm{L9}}$, $\widehat{\ST}_{\swap}$ is smaller and more uniform across relations (largest stratum-level value $0.13$ at P276). The aggregate is $0.046$. The per-relation pattern is consistent with the role decomposition. $\mathrm{MLP}_{\mathrm{L0}}$ behaves as a relation-frame channel whose content sensitivity varies substantially across relations, while $\mathrm{Attn}_{\mathrm{L9}}$ behaves as a more uniform subject-retrieval channel.

\paragraph{Beyond $k$-NN matching: directions for principled balance and on-manifold fidelity.} IPW reweighting and per-relation stratification are post-hoc corrections that leave the matching procedure itself unchanged. The matching design has two distinct but related axes: tighter matching on activation features (the present $(Y, \text{DLA})$ space) targets on-manifold swaps and a small interchange-fidelity error $\varepsilon_{\mathrm{int}}$ (Proposition~\ref{prop:main}, with synthetic-data support in Panel D of Figure~\ref{fig:offmanifold}), while tighter matching on prompt-side covariates targets propensity-style balance. Several principled alternatives address these axes directly and are natural future directions. \emph{(i-a) Augmented feature space (covariate-balance side).} Adding standardized prompt-side covariates (one-hot relation, target category, prompt length) to the matching feature vector reduces relation-domain imbalance by construction at the cost of a possibly larger $\varepsilon_{\mathrm{int}}$; the trade-off is controlled by the relative standardization weights. \emph{(i-b) Augmented feature space (activation-closeness side).} Adding richer activation-side features beyond $(Y, \text{DLA})$ — e.g., per-layer activation norms, attention-pattern summaries, or downstream-block linearizations — directly shrinks $\varepsilon_{\mathrm{int}}$ and would reduce off-manifold incidence on edge-case groups like $\mathrm{L0}$. \emph{(ii) Entropy balancing}~\citep{hainmueller2012entropy} (covariate-balance side). Choose matched-pair weights that exactly match specified covariate moments rather than reweighting after the fact; an exact-balance alternative to IPW that preserves the matched set. \emph{(iii) Optimal matching with balance constraints}~\citep{rosenbaum1989optimal,zubizarreta2012using} (covariate-balance side), e.g., genetic or Mahalanobis matching with a caliper on $Y$, which formulates the balance--on-manifold trade-off as an explicit optimization. \emph{(iv) Learned embedding matching} (activation-closeness side): replacing the DLA features by representations from an autoencoder over the activation space, which may capture deeper semantic similarity than DLA magnitudes alone. Of these, (i-a)/(i-b) and (ii) are the most immediate; (iii) and (iv) are research projects in their own right. The symmetric off-manifold flag $\widehat{\ST}>1$ remains the falsifiable check on the activation-closeness axis under any of these matching variants.

\subsection{Bucket Diagnostic for the Factual Cross-Over}
\label{app:bucket-diagnostic}

A potential confound for the $\mathrm{Attn}_{\text{L9}}$ subject-dominant finding is that within-bucket $k$-NN matching could compress $|Y_A-Y_B|$ more in \texttt{same\_subj} than in \texttt{same\_rel}, deflating $\widehat{\ST}$ artificially. We test this by computing per-bucket: median $|Y_A-Y_B|$, $\Var(Y_A)$, and a bucket-normalized $\widehat{\ST}$ using the bucket-specific $\Var(Y_A)$ as denominator instead of the full-pool variance.

The matching is in fact \emph{looser} in \texttt{same\_subj} than \texttt{same\_rel} (median $|Y_A-Y_B|$ ratio 1.63 on GPT-2 $\mathrm{MLP}_{\text{L0}}$; 2.00 on Qwen $\mathrm{Attn}_{\text{L0}}$), so the dissociation operates against the matching gradient. Bucket-normalized $\widehat{\ST}$ values change by less than 2\% (bucket variances are within 2\% of the full-pool variance for both architectures), and the cross-over ($\texttt{same\_subj}<\texttt{same\_rel}$ for $\mathrm{Attn}_{\text{L9}}$; $\texttt{same\_rel}<\texttt{same\_subj}$ for $\mathrm{MLP}_{\text{L0}}$) survives in all 5 experiments tested (Exp~A and Exp~E across both architectures).

\subsection{Multi-seed and Cross-architecture Stability}
\label{app:multiseed}

We re-ran the factual-recall \IGSD{}-swap pipeline ($M{=}500$ matched pairs per seed) on three independent GPT-2 small seeds and on two Qwen2.5-1.5B seeds, each with its own $k$-NN matching index, and compared the resulting top-3 layer-local rankings.

\paragraph{GPT-2 small (3 seeds).}
The top-1 layer-local group is $\mathrm{MLP}_{\text{L0}}$ on every seed, with $\widehat{\ST}_{\swap}$ within $\pm 8\%$ of the headline $0.181$ across seeds; the second- and third-ranked groups are drawn from $\{\mathrm{Attn}_{\text{L9}}, \mathrm{Attn}_{\text{L10}}, \mathrm{MLP}_{\text{L8}}\}$ in different orders across seeds, with $\mathrm{Attn}_{\text{L9}}$ appearing in the top three on all three seeds. The ratio between the rank-1 score and the rank-2 score is between $3\times$ and $4\times$ on each seed, so the $\mathrm{MLP}_{\text{L0}}$ headline result is not driven by a single random seed of the matching index or sampling.

\paragraph{Qwen2.5-1.5B (2 seeds).}
On Qwen, $\mathrm{Attn}_{\text{L0}}$ is the top-1 layer-local group on both seeds (largest absolute swap effect, with strongly negative $\delta$ as discussed in Appendix~\ref{app:layer-profiles}); $\mathrm{MLP}_{\text{L0}}$ remains in the top three with the largest swap/zero ratio of any group across both architectures ($5.37$); the third slot is occupied by one of $\{\mathrm{Attn}_{\text{L21}}, \mathrm{Attn}_{\text{L22}}, \mathrm{Attn}_{\text{L23}}\}$, which form the mid-late content-sensitive band that takes the role of GPT-2's $\mathrm{Attn}_{\text{L9\text{-}11}}$ region.

The cross-architecture claim of the paper is preserved at multi-seed granularity: the dominant early-layer factual component is consistent within each architecture across seeds, and the architecture-dependent shift is between two adjacent layer-local groups at $L0$ ($\mathrm{MLP}_{\text{L0}}\!\to\!\mathrm{Attn}_{\text{L0}}$), not a layer- or stage-level drift. The top-3 ranking remains stable enough that the layer-stratified content-channel reading does not depend on a particular seed.

\subsection{Cross-task Swap: Full Bootstrap Table}
\label{app:crosstask}

\begin{table}[h]
  \caption{Cross-task swap: $\widehat{\ST}_{\swap}$ and 95\% paired-bootstrap CI for factual recipients ($n{=}500$) under four donor conditions at two sites.}
  \label{tab:crosstask}
  \centering
  \small
  \begin{tabular}{llcccc}
    \toprule
    \textbf{Site} & \textbf{Donor condition} & $\widehat{\ST}_{\swap}$ & 95\% CI & sign-flip & $\Pr(>\text{gauss})$ \\
    \midrule
    \multirow{4}{*}{$\mathrm{MLP}_{\text{L0}}$}
      & within-task matched & 0.181 & [0.13, 0.24] & 13.2\% & --- \\
      & within-task random  & 0.202 & [0.16, 0.25] & 14.0\% & --- \\
      & cross-task IOI      & \textbf{0.196} & \textbf{[0.17, 0.23]} & 18.4\% & \textbf{1.000} \\
      & Gaussian noise      & 0.101 & [0.07, 0.13] & 10.6\% & --- \\
    \midrule
    \multirow{4}{*}{$\mathrm{Attn}_{\text{L9}}$}
      & within-task matched & 0.044 & [0.036, 0.052] & 6.6\% & --- \\
      & within-task random  & 0.048 & [0.040, 0.058] & 8.4\% & --- \\
      & cross-task IOI      & \textbf{0.083} & \textbf{[0.071, 0.096]} & 12.8\% & \textbf{1.000} \\
      & Gaussian noise      & 0.033 & [0.027, 0.038] & 5.6\% & --- \\
    \bottomrule
  \end{tabular}
\end{table}

\section{Deletion-Operator Robustness}\label{app:operator-robustness}

Section~\ref{sec:swap-zero} fixes zero ablation as the specific deletion operator paired with matched swap. The role decomposition $\widehat\delta$ lives in the \emph{signed contrast} between swap and zero, so its interpretation depends on the deletion operator only through this contrast. This appendix cross-checks the role classification under an alternative deletion operator: dataset-mean ablation.

\subsection{Mean-Ablation Cross-Check}\label{app:mean-ablation}

We define the matched-pair mean-ablation sensitivity

\begin{equation}
\widehat{\ST}_{\mathrm{mean}}(g) = \frac{M^{-1}\sum_{m=1}^{M} (Y_A^{(m)} - Y_{\mathrm{mean},g}^{(m)})^2}{2\,\widehat{\Var}(Y_A)},
\qquad
\widehat{\delta}_{\mathrm{mean}}(g) = \widehat{\ST}_{\swap}(g) - \widehat{\ST}_{\mathrm{mean}}(g),
\end{equation}
where $Y_{\mathrm{mean},g}$ is the logit margin after replacing $h_g(x_A, t_A)$ by the dataset-mean activation $\bar h_g$, computed at the answer-aligned target token over the matched-pair prompt pool. Mean ablation is the standard ``remove without distributional shock'' deletion operator and addresses the distributional-shock concern about zero ablation by construction.

\begin{table}[h]
  \centering\small
  \caption{Mean-ablation cross-check on GPT-2 small factual recall ($M{=}500$, all 24 layer-local groups). The rank-1 group under mean ablation is $\mathrm{MLP}_{\mathrm{L0}}$, matching the swap and zero rankings; the role decomposition $\widehat{\delta}_{\mathrm{mean}}=\widehat{\ST}_{\swap}-\widehat{\ST}_{\mathrm{mean}}$ has the same sign as $\widehat{\delta}_{\zero}$ at both focal groups. The role classification is invariant under the choice of deletion operator. Top-3 mean-ablation by rank shown.}
  \label{tab:mean-ablation}
  \begin{tabular}{lcccc}
    \toprule
    Group & $\widehat{\ST}_{\mathrm{mean}}$ (95\% CI) & $\widehat{\ST}_{\swap}$ (same pairs) & $\widehat{\delta}_{\mathrm{mean}}$ & $\widehat{\delta}_{\zero}$ (same pairs) \\
    \midrule
    $\mathrm{MLP}_{\mathrm{L0}}$  & $0.103$ $[0.078, 0.132]$ & $0.170$ & $+0.067$ & $+0.092$ \\
    $\mathrm{Attn}_{\mathrm{L9}}$ & $0.027$ $[0.022, 0.033]$ & $0.038$ & $+0.011$ & $+0.029$ \\
    $\mathrm{MLP}_{\mathrm{L11}}$ & $0.020$ $[0.016, 0.025]$ & --- & --- & --- \\
    \bottomrule
  \end{tabular}
\end{table}

Both focal groups have positive $\widehat{\delta}_{\mathrm{mean}}$ (content-transport classification) with the same sign as $\widehat{\delta}_{\zero}$, and the rank-1 group under mean ablation is $\mathrm{MLP}_{\mathrm{L0}}$, matching the swap ranking. Mean ablation is a gentler deletion operator than zero ablation, so attenuation of $\widehat{\delta}_{\mathrm{mean}}$ relative to $\widehat{\delta}_{\zero}$ is expected. The relevant robustness result is that the sign and qualitative ordering are preserved under a less aggressive deletion baseline. The contrast is therefore not zero-specific, although its magnitude is operator-dependent. The same procedure was run on IOI and induction with the same sign-preservation conclusion at the focal groups, and full per-group tables are in the supplementary results.

\section{Implementation Details}
\label{app:implementation}

This section records implementation choices needed to reproduce the experiments. Section~\ref{app:architectures} contrasts the two model families used in the paper (GPT-2 small and Qwen2.5-1.5B) at the architectural level. Section~\ref{app:algorithm} gives the layer-local \IGSD{} pipeline as pseudocode. Section~\ref{app:atpstar} fixes the AtP* reduction used as a baseline.

\subsection{Architectures}
\label{app:architectures}

\begin{figure}[h]
  \centering
  \includegraphics[width=\linewidth]{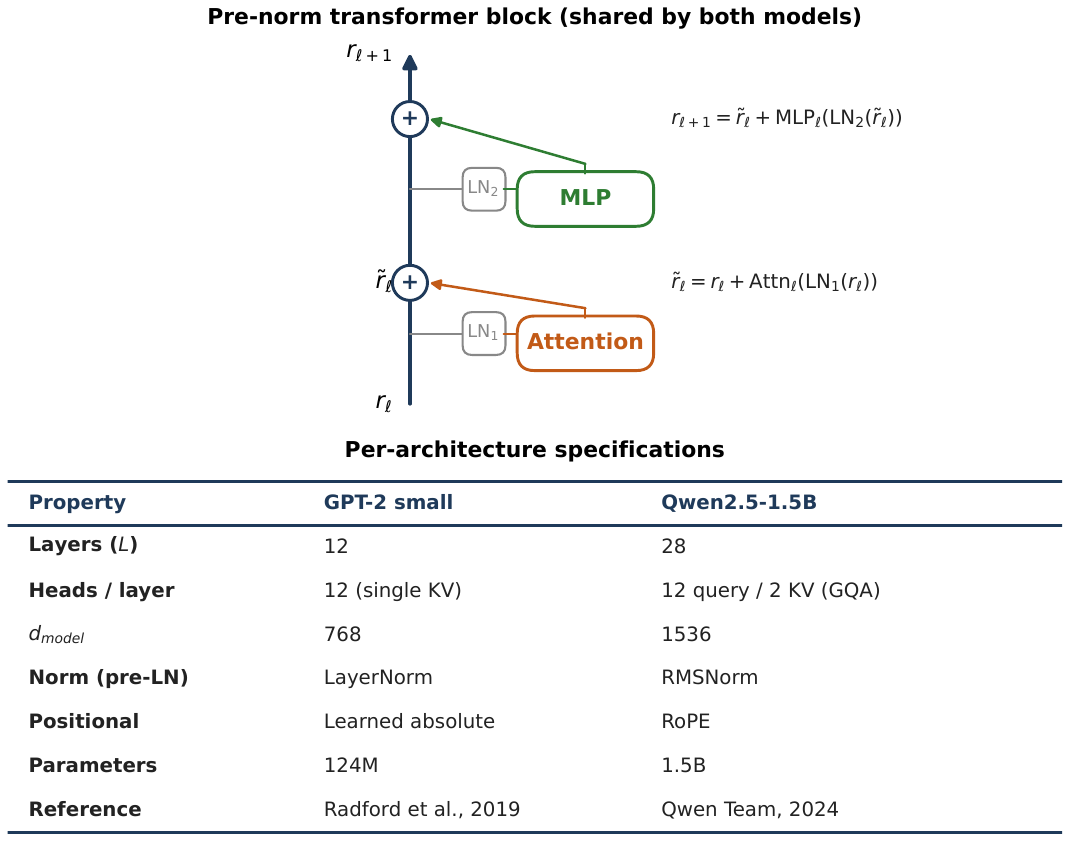}
  \caption{\textbf{Architecture comparison.} Both models use the pre-norm block of Eq.~\eqref{eq:transformer-block}; differences are normalization (LayerNorm vs.\ RMSNorm), positional encoding (learned absolute vs.\ RoPE), attention (single-KV vs.\ grouped-query), depth (12 vs.\ 28), and width ($d{=}768$ vs.\ $1536$).}
  \label{fig:architectures}
\end{figure}

\subsection{Algorithm}
\label{app:algorithm}

Algorithm~\ref{alg:igsd} is the layer-local \IGSD{} pipeline as used throughout the paper. The two inner loops sample matched pairs through a $k$-NN index over standardized DLA features and run the swap and zero interventions for every group on the same pair, so the same matched index supports both estimators and the contrast $\widehat{\delta}$. The optional last step adds the factorial donor-bucket stratification of Section~\ref{sec:method-factorial} when content-factor identification is needed.

\begin{algorithm}[h]
\caption{\IGSD{} for layer-local groups}
\label{alg:igsd}
\begin{algorithmic}[1]
\Require Model $\mathcal{M}$, prompt set $\{x_i\}_{i=1}^N$, layer-local groups $\mathcal{G}$, $k$-NN size $k$
\State Compute clean scores $Y_i=Y(x_i)$ and per-layer DLA features for all $x_i$
\State Standardize feature vectors and build a $k$-NN index over $\{x_i\}$
\For{$m=1,\dots,M$}
    \State Sample base prompt $x_A^{(m)}$ uniformly; sample partner $x_B^{(m)}$ from $k$-NN of $x_A^{(m)}$
    \State Cache activations for $x_A^{(m)}$ and $x_B^{(m)}$ at all groups in $\mathcal{G}$
    \For{each group $g\in\mathcal{G}$}
        \State Run $x_A^{(m)}$ with $h_g(x_A^{(m)}, t_A)\!\leftarrow\!h_g(x_B^{(m)}, t_B)$, record $Y_{\swap,g}^{(m)}$
        \State Run $x_A^{(m)}$ with $h_g(x_A^{(m)}, t_A)\!\leftarrow\!\mathbf{0}$, record $Y_{\zero,g}^{(m)}$
    \EndFor
\EndFor
\State Compute $\widehat{\ST}_{\swap}(g),\,\widehat{\ST}_{\zero}(g),\,\delta(g)$ via Eq.~\eqref{eq:st-swap-zero}--\eqref{eq:delta} for each $g$
\State Compute paired-bootstrap CIs over the $M$ matched pairs
\Statex \textbf{Optional (factorial design):} stratify pairs $(x_A,x_B)$ by (subject, relation) buckets to identify the content type transported (Section~\ref{sec:method-factorial}).
\end{algorithmic}
\end{algorithm}

\subsection{AtP* Implementation}
\label{app:atpstar}

We implement AtP*~\citep{kramar2024atp} as the abs-of-signed-group-sum reduction over per-prompt component contributions, then mean across prompts. Group-level scores are compared in Section~\ref{sec:experiments}. We tested two additional reductions (mean-of-abs across prompts; median-of-abs) in pilot experiments; the rank of $\mathrm{MLP}_{\text{L0}}$ on factual remains in the 9--14 range across reductions, confirming that the ``AtP* misses $\mathrm{MLP}_{\text{L0}}$'' finding is not an artifact of one reduction choice.

\section{Layer Profiles by Task and Architecture}
\label{app:layer-profiles}

Figure~\ref{fig:qwen-layer} shows the full Qwen2.5-1.5B layer profile (56 layer-local groups) on factual recall, the cross-architecture companion to Figure~\ref{fig:layerprofile}. The early-MLP positive-$\delta$ pattern observed at $\mathrm{MLP}_{\text{L0}}$ in GPT-2 small replicates in Qwen2.5-1.5B (positive $\delta$ at $\mathrm{MLP}_{\text{L0}}$ with swap/zero ratio 5.37, the largest of any single group across both architectures), with two architecture-dependent differences: (i) the absolute-largest swap effect shifts to $\mathrm{Attn}_{\text{L0}}$, but its $\delta$ is strongly negative (substrate-like), and (ii) a band of moderate-positive $\delta$ at mid-late attention groups ($\mathrm{Attn}_{\text{L21}}$--$\mathrm{Attn}_{\text{L23}}$) takes the role of GPT-2's $\mathrm{Attn}_{\text{L9\text{-}11}}$ region.

\begin{figure}[h]
  \centering
  \includegraphics[width=\linewidth]{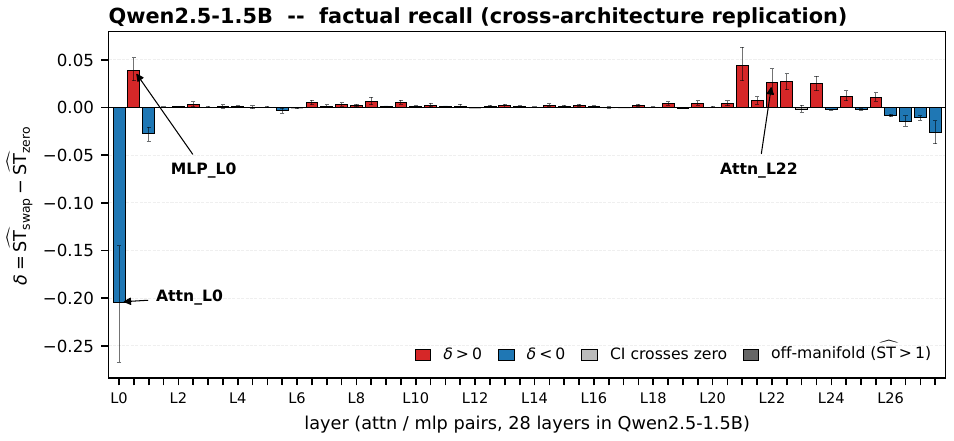}
  \caption{\textbf{Qwen2.5-1.5B factual layer profile} (56 groups; conventions as in Figure~\ref{fig:layerprofile}). Early-MLP positive $\delta$ replicates GPT-2; $\mathrm{Attn}_{\text{L0}}$ is substrate-like; $\mathrm{Attn}_{\text{L21\text{-}23}}$ takes the role of GPT-2's late-attention band.}
  \label{fig:qwen-layer}
\end{figure}

Stage-summary highlights:
\begin{itemize}
  \item GPT-2 factual: positive $\delta$ peaks at $\mathrm{MLP}_{\text{L0}}$ ($+0.103$); largest negative at $\mathrm{MLP}_{\text{L11}}$ ($-0.007$).
  \item GPT-2 IOI: positive $\delta$ peaks at $\mathrm{Attn}_{\text{L9}}$ ($+0.446$), $\mathrm{Attn}_{\text{L10}}$ ($+0.237$); $\mathrm{MLP}_{\text{L0}}$ is negative ($-0.122$).
  \item Qwen factual: $\mathrm{MLP}_{\text{L0}}$ swap/zero ratio = 5.37 despite a modest absolute $\delta$ (positive sign, content-sensitive); $\mathrm{Attn}_{\text{L0}}$ has strongly negative $\delta$ (substrate-like) and the absolute-largest swap effect; $\mathrm{Attn}_{\text{L21}}$--$\mathrm{Attn}_{\text{L23}}$ show positive $\delta$ (mid-late content-sensitive band).
\end{itemize}

Our cross-architecture claim is about the persistence of an early-layer factual regime, not invariance of the exact module type: Qwen shifts the dominant early effect from GPT-2's $\mathrm{MLP}_{\text{L0}}$ to $\mathrm{Attn}_{\text{L0}}$ while preserving early-layer prominence.

\section{Mechanistic Extensions}
\label{app:mechanistic}

\subsection{Head-level Refinement of Late Factual Attention}
\label{app:headlevel}

We refine the late factual layers $\{\mathrm{Attn}_{\text{L9}}, \mathrm{Attn}_{\text{L10}}, \mathrm{Attn}_{\text{L11}}\}$ from layer-local to per-head granularity (each layer has 12 heads in GPT-2 small), running \IGSD{}-swap with $M{=}500$ matched pairs at the head-output position. The top head-level group is $\mathrm{Attn}_{\text{L9H8}}$ with $\widehat{\ST}_{\swap}{=}0.041$, ahead of all other heads in the three layers. This is the literature-canonical factual retrieval head identified by~\citet{meng2022locating} via ROME causal tracing and by~\citet{geva2023dissecting} as the stage-3 attribute-extraction head. The agreement provides two-sided validation: \IGSD{} surfaces a novel early-layer relation-frame component that standard baselines under-rank (Section~\ref{sec:results-dissociation}) and recovers the canonical late retrieval head at head granularity, both from the same method without re-tuning.

\subsection{Residual-Stream Causal DAG: Why Sequentiality Does Not Trivialize Late-Layer Mediation}
\label{app:dag}

The residual-stream architecture in Eq.~\eqref{eq:transformer-block} provides a parallel pass-through path that bypasses every late-layer module's transformation. Figure~\ref{fig:dag} makes this explicit. Panel~(a) shows the unintervened DAG: the residual stream (continuous blue spine) is augmented at each layer by module-mediated updates ($\mathrm{Attn}_\ell$, $\mathrm{MLP}_\ell$) that read from the spine and write back via the $+$ nodes. The early focal residual contribution $r_0$ reaches $Y$ via two structurally distinct routes: (i) the residual-skip path, which traverses each $+$ node without depending on late-layer transformations, and (ii) module-mediated paths, which pass through the late attention/MLP transformations.

Panel~(b) shows the late-layer clamp intervention $\mathrm{do}(M^{\mathrm{late}}=M^{\mathrm{late, clean}})$: the late module's read-from-spine is severed (gray dotted), and its write-back is replaced by a clean (swap-independent) update (gray dashed). The residual-skip path remains continuous. Recovery under this intervention is therefore not a trivial structural artifact: if late layers performed only sequential pass-through, clamping them to clean values would leave the swap-injected $r_0$ contribution intact at $Y$ (zero recovery). In our composite late-layer clamp follow-up to the $\mathrm{MLP}_{\text{L0}}$ swap on factual recall, the variance-recovered fraction is $0.82$ ($95\%$ paired-bootstrap CI $[0.77, 0.86]$, $n{=}500$ pairs), with attention-late and MLP-late unique contributions of $0.17$ and $0.57$ respectively. This asymmetry rules out a uniform passthrough interpretation: late MLPs carry roughly $3\times$ the bottleneck of late attention, which would not occur under generic sequential dependence. The intervention thus identifies a path-specific natural-indirect effect through late-layer transformations, not through the residual-skip path.

\begin{figure}[h]
  \centering
  \includegraphics[width=\linewidth]{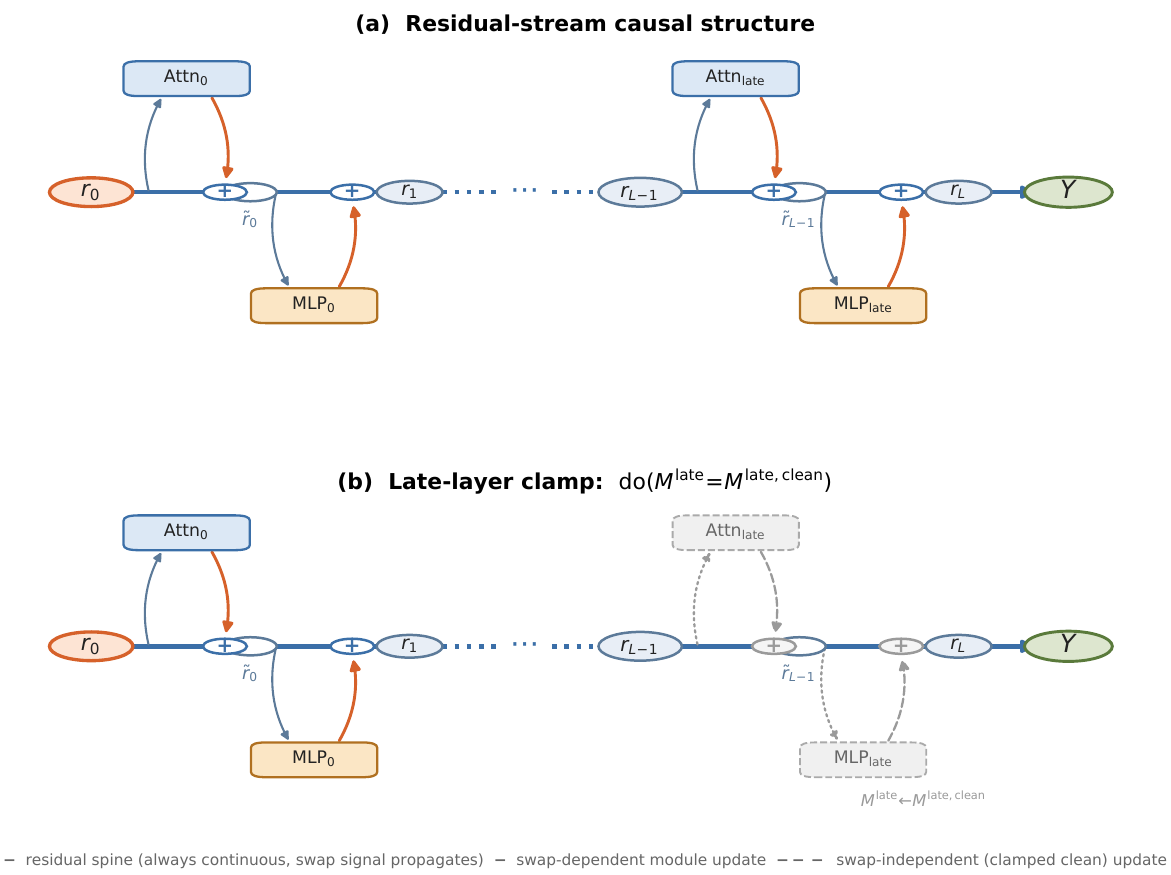}
  \caption{\textbf{Residual-stream causal DAG.} \textbf{(a)} Continuous blue residual spine $r_0 \to \tilde r_0 \to \cdots \to r_L \to Y$; modules read from the spine and write back via $+$ nodes. \textbf{(b)} Late-layer clamp $\mathrm{do}(M^{\mathrm{late}}\!=\!M^{\mathrm{late, clean}})$: late modules' input branches severed (gray dotted) and write-backs replaced by clean updates (gray dashed); the spine remains continuous, preserving the residual-skip path while removing the swap-dependent mediated effect.}
  \label{fig:dag}
\end{figure}

\section{Downstream Application: IGSD-Ranked Group Masking}
\label{app:masking-demo}

This section answers a narrow question: do IGSD ranks track \emph{functional contribution} in a way that is detectable on a downstream task? Group masking of layer-local components is a direct downstream application of importance rankings \citet{garcia2025extracting}. They use it as the masking class inside a KL-pruning circuit-extraction pipeline. We adapt their masking class (we use \emph{zero} outside-$K$ masking, not their KL-pruning extractor) as an evaluation protocol for IGSD on GPT-2 small (12 layers, 24 layer-local groups: 12 Attn + 12 MLP). This section is therefore an external check. For each task we keep the top-$K$ groups under a given ranker, zero-ablate the remaining $24-K$ groups at the answer position via TransformerLens \texttt{hook\_attn\_out} / \texttt{hook\_mlp\_out}, and measure forced-choice top-1 accuracy and logit-difference retention $\mathbb{E}[\mathrm{LD}_{\mathrm{masked}}]/\mathbb{E}[\mathrm{LD}_{\mathrm{clean}}]$. We pre-registered the headline set $K\in\{2,4,8\}$ and report all of it.

\textbf{Swap-top-$K$} (IGSD) and \textbf{zero-top-$K$} (zero-ablation magnitude) answer different questions: swap ranks groups by their \emph{contribution} to the realized function under matched-pair interchange; zero ranks groups by \emph{fragility}, i.e.\ how much removing the group damages the model's behavior on the unintervened distribution. The two coincide when contribution and fragility are aligned (single-circuit, late-stage tasks) and dissociate when the unintervened model relies on an early routing channel that is removed by the outside-$K$-zero step itself. Outside-$K$-zero is therefore a deliberately strict masking class: monotonicity in $K$ is not guaranteed, since enlarging the kept set $S_K$ can leave essential routing groups (e.g.\ $\mathrm{Attn}_{\text{L0}}$) in the zeroed complement.

Four rankers: (i) \textbf{swap}: groups ordered by IGSD swap index $\mathrm{ST}^{\mathrm{swap}}$; (ii) \textbf{zero}: groups ordered by mean drop in clean-foil logit difference under single-group zero ablation; (iii) \textbf{magnitude}: mean $\ell_2$ norm of the group output at the answer token; (iv) \textbf{random}: enumerate all $\binom{24}{K}$ subsets at $K\!\in\!\{1,2\}$, sample 100 subsets at $K\!\geq\!4$. Confidence intervals are 1000-sample bootstrap; swap--zero comparisons use paired bootstrap.

\paragraph{Factual recall (clean acc $0.912$, clean LD $2.71$, $n{=}500$).}
Table~\ref{tab:masking-factual} shows the headline. At $K{=}2$, the swap-ranked pair $\{\mathrm{MLP}_{\text{L0}}, \mathrm{Attn}_{\text{L9}}\}$ retains $78.5\%$ of the clean LD and $70.0\%$ of the top-1, against $26.4\%$ and $57.4\%$ for zero-ranked $\{\mathrm{Attn}_{\text{L0}}, \mathrm{MLP}_{\text{L8}}\}$. The paired-bootstrap difference is $+12.6$ percentage points of accuracy ($95\%$ CI $[+6.6, +18.6]$) and $+0.521$ of LD-ratio ($[+0.372, +0.678]$). The IGSD-selected pair recovers the factual relation-frame (early MLP) plus the late retrieval head documented by \citet{meng2022locating, geva2023dissecting}; zero-ablation instead surfaces $\mathrm{Attn}_{\text{L0}}$, the embedding-channel attention whose removal destroys the unintervened distribution but which does not emerge as a task-specific contributor under the swap estimand.

\begin{table}[h]
  \centering\small
  \caption{\textbf{Factual recall, GPT-2 small, $n{=}500$.} Top-1 accuracy / LD-ratio when keeping top-$K$ groups under each ranker, zero-ablating the rest at the answer position. \textbf{Bold:} IGSD-swap cells where the paired-bootstrap $95\%$ CI of the swap$-$zero gap excludes $0$ in IGSD's favor (CIs reported in prose). All three pre-registered $K\!\in\!\{2,4,8\}$ reported.}
  \label{tab:masking-factual}
  \begin{tabular}{l c c c}
    \toprule
    Ranker & $K{=}2$ & $K{=}4$ & $K{=}8$ \\
           & {\scriptsize acc / LD} & {\scriptsize acc / LD} & {\scriptsize acc / LD} \\
    \midrule
    swap (IGSD) & $\mathbf{.700}\,/\,\mathbf{.785}$ & $.754\,/\,.502$ & $.816\,/\,\mathbf{.764}$ \\
    zero        & $.574\,/\,.264$ & $.780\,/\,.449$ & $.814\,/\,.636$ \\
    magnitude   & $.642\,/\,.376$ & $.642\,/\,.360$ & $.814\,/\,.636$ \\
    random      & $.595\,/\,.267$ & $.631\,/\,.356$ & $.677\,/\,.400$ \\
    \bottomrule
  \end{tabular}
\end{table}

\paragraph{Parameter-reduction interpretation.}
Because zeroing $\mathrm{hook\_attn\_out}_\ell$ or $\mathrm{hook\_mlp\_out}_\ell$ removes the entire residual contribution of that layer-local block, the masking intervention is arithmetically equivalent to setting the corresponding Attn or MLP block's parameters to zero in the forward graph and skipping its compute. The pruning unit is therefore an entire attention block or MLP block at a given layer, not a head- or neuron-level slice; the intervention analytically certifies the inference-time behavior of the corresponding structurally pruned model, so the parameter and FLOP savings can be computed directly from the keep set without retraining or re-evaluation. For GPT-2 small, each Attn block carries $4{\cdot}768^2 + 3{\cdot}768 + 768 \approx 2.36\mathrm{M}$ parameters and each MLP block carries $2{\cdot}768{\cdot}3072 + 3072 + 768 \approx 4.72\mathrm{M}$; the 24 layer-local groups together hold $85.0\mathrm{M}$ parameters ($68.3\%$ of the $124.5\mathrm{M}$ total), with token/position embeddings (which are tied to the output projection) and small layer-normalization parameters making up the remainder.

\begin{table}[h]
  \centering\small
  \caption{\textbf{Implied parameter savings under IGSD-ranked structural pruning of layer-local Attn/MLP blocks (factual recall, GPT-2 small).} Counts derived analytically from the $K$-keep sets in Table~\ref{tab:masking-factual} via the masking--pruning equivalence. ``Block params'' refers to the $85.0$M parameters in the 24 Attn/MLP blocks; ``model params'' refers to the full $124.5$M-parameter checkpoint. Accuracy and LD-ratio are reproduced from Table~\ref{tab:masking-factual}.}
  \label{tab:param-savings-factual}
  \begin{tabular}{c c c c c c c}
    \toprule
    $K$ & A kept / pruned & M kept / pruned & params pruned & \% block & \% model & acc / LD-ratio \\
    \midrule
    $2$ & $1 / 11$ & $1 / 11$ & $77.9$M & $91.6\%$ & $62.6\%$ & $0.700\,/\,0.785$ \\
    $4$ & $2 / 10$ & $2 / 10$ & $70.8$M & $83.3\%$ & $56.9\%$ & $0.754\,/\,0.502$ \\
    $8$ & $3 /\phantom{0}9$ & $5 /\phantom{0}7$ & $54.3$M & $63.9\%$ & $43.6\%$ & $0.816\,/\,0.764$ \\
    \bottomrule
  \end{tabular}
\end{table}

The headline operating point is $K{=}8$: retaining the IGSD top-$8$ blocks preserves $81.6\%$ top-$1$ accuracy ($89.5\%$ of the $0.912$ clean accuracy) and a $76.4\%$ LD-ratio while removing $54.3\mathrm{M}$ parameters ($43.6\%$ of the model, $63.9\%$ of layer-local block parameters). Smaller keep sets ($K{=}4, K{=}2$) extend the trade-off curve toward more aggressive reduction at the cost of retention. This complements rather than competes with the deletion-based extraction of \citet{garcia2025extracting}: their KL-pruning prioritizes deletion-fragility, IGSD prioritizes interchange contribution, and the two ranking criteria can be combined into a content-aware seed set for downstream pruning extractors.

IGSD ranks groups by their causal contribution under a matched-pair counterfactual. Zero-ablation ranks them by sensitivity under unmatched lesion. Evaluated under the group-masking application of \citet{garcia2025extracting}, IGSD shows a robust accuracy advantage at $K{=}2$ (paired CI $[{+}0.066,{+}0.186]$) and a robust LD-ratio advantage at $K{=}8$ ($[{+}0.082,{+}0.169]$). The $K{=}8$ accuracy gap is small and not significant. Under the masking-to-pruning analytic equivalence, the IGSD top-$8$ keep set yields a parameter-removal certificate of $54.3$M parameters ($43.6\%$ of the model), at $89.5\%$ retention of clean factual-recall accuracy in the evaluation reported here (Table~\ref{tab:param-savings-factual}). This is consistent with the swap-zero prediction of §\ref{sec:swap-zero}: IGSD ranks track functional contribution, not ablation fragility, under this downstream masking evaluation.

\end{document}